\documentclass[12pt,a4paper]{article}
\usepackage{microtype}
\usepackage{graphicx}
\usepackage{verbatim}
\usepackage{xcolor}
\usepackage{subcaption}
\usepackage{hyperref}
\usepackage{amsmath}
\usepackage{amssymb}
\usepackage{mathtools}
\usepackage{amsthm}
\usepackage{cancel}
\usepackage[T1]{fontenc}
\usepackage[numbers]{natbib}
\usepackage{cleveref}
\crefformat{section}{\S\hspace{0pt}\textnormal{#2#1#3}}
\crefformat{theorem}{Theorem~\textnormal{#2#1#3}}
\crefformat{corollary}{Corollary~\textnormal{#2#1#3}}
\crefformat{definition}{Definition~\textnormal{#2#1#3}}
\crefformat{figure}{Fig.~\textnormal{#2#1#3}}
\crefformat{example}{Example~\textnormal{#2#1#3}}

\newcommand{\dd}{\mathrm{d}}

\newcommand{\SO}{\mathrm{SO}}

\newcommand{\R}{\mathbb{R}}

\theoremstyle{plain}
\newtheorem{theorem}{Theorem}[section]

\newtheorem{corollary}[theorem]{Corollary}
\theoremstyle{definition}
\newtheorem{definition}[theorem]{Definition}
\newtheorem{example}[theorem]{Example}

\theoremstyle{remark}
\newtheorem{remark}[theorem]{Remark}

\title{Equivariant Manifold Neural ODEs\\and Differential Invariants}
\date{}

\begin{document}

\maketitle

\vspace*{-10mm}
\begin{center}
\begin{minipage}[t]{0.8\textwidth}
\textbf{Emma Andersdotter} \hfill {\footnotesize \tt emma.andersdotter@umu.se}\\
\textit{\footnotesize Department of Mathematics and Mathematical Statistics\\
Ume{\aa} University\\ Ume{\aa} SE-901 87, Sweden}
\end{minipage}
\end{center}

\begin{center}
\begin{minipage}[t]{0.8\textwidth}
\textbf{Daniel Persson} \hfill {\footnotesize \tt daniel.persson@chalmers.se} \\
\textit{\footnotesize Department of Mathematical Sciences\\ Chalmers University of Technology and Gothenburg University\\Gothenburg SE-412 96, Sweden}
\end{minipage}
\end{center}

\begin{center}
\begin{minipage}[t]{0.8\textwidth}
\textbf{Fredrik Ohlsson}  \hfill {\footnotesize \tt fredrik.ohlsson@umu.se}\\
\textit{\footnotesize Department of Mathematics and Mathematical Statistics\\
Ume{\aa} University\\ Ume{\aa} SE-901 87, Sweden}
\end{minipage}
\end{center}
\vspace*{5mm}

\begin{abstract}
\noindent In this paper, we develop a manifestly geometric framework for equivariant manifold neural ordinary differential equations (NODEs) and use it to analyse their modelling capabilities for symmetric data. First, we consider the action of a Lie group $G$ on a smooth manifold $M$ and establish the equivalence between equivariance of vector fields, symmetries of the corresponding Cauchy problems, and equivariance of the associated NODEs. We also propose a novel formulation, based on Lie theory for symmetries of differential equations, of the equivariant manifold NODEs in terms of the differential invariants of the action of $G$ on $M$, which provides an efficient parameterisation of the space of equivariant vector fields in a way that is agnostic to both the manifold $M$ and the symmetry group $G$. Second, we construct augmented manifold NODEs, through embeddings into flows on the tangent bundle $TM$, and show that they are universal approximators of diffeomorphisms on any connected $M$. Furthermore, we show that universality persists in the equivariant case and that the augmented equivariant manifold NODEs can be incorporated into the geometric framework using higher-order differential invariants. Finally, we consider the induced action of $G$ on different fields on $M$ and show how it can be used to generalise previous work, on, e.g., continuous normalizing flows, to equivariant models in any geometry. 
\end{abstract}

\newpage
\tableofcontents
\section{Introduction}\label{sec:introduction}
Recent years have seen a 'geometrisation' of neural networks, driven by the need to deal with data defined on non-Euclidean domains, for example, graphs or manifolds. The field of geometric deep learning~\cite{Bronstein2017,Bronstein2021,Gerken2020} aims to incorporate the geometry of data in a foundational mathematical description of neural networks. This endeavour has been extremely successful in the context of  group equivariant convolutional neural networks (CNNs), in which the symmetries of the underlying data are built into the neural network using techniques from group theory and representation theory. 

Another development in the same spirit explores the connection to differential equations and dynamical systems in the limit of infinitely deep networks. When formulated in terms of the dynamics propagating information through the network, the learning problem becomes amenable to powerful numerical techniques for differential equations and a rich theory of dynamical systems. In particular, neural ordinary differential equations (NODEs) have received considerable attention since they were proposed in their current form in~\cite{Chen2018}. 

NODE models were originally conceived by considering the continuous dynamical systems obtained in the limit of infinitely deep residual neural networks~\cite{He2016}, building on previous similar constructions considered in~\cite{E2017,Haber2018,Ruthotto2020}. The dynamical system describes the evolution of a state $u(t) \in \mathbb{R}^n$ according to a governing ordinary differential equation (ODE)  
\begin{equation}
\label{eq:NODE-Rn}
    \dot{u}(t) = \phi(u,t) \,,
\end{equation}
where $\phi : \R^n \times \R \to \R^n$ is a vector valued function parameterising the dynamics. The NODE model is then given by the map $h : \R^n \to \R^n$ defined by evolving the input $x=u(0)$ to the output $h(x) = u(1)$. In $\R^n$, this can be obtained by direct integration of the vector field $\phi$ as
\begin{equation*}
    h(x) = x + \int_0^1 \phi(u,t) \, \dd t \,.
\end{equation*}
In~\cite{Chen2018}, $\phi$ was parameterised using a fully connected feed-forward neural network, and backpropagation of gradients by solving an adjoint ODE related to~\eqref{eq:NODE-Rn} was used to enable learning $\phi$, extending previous results in~\cite{stapor2018}.

NODE models possess several attractive theoretical properties. Of particular conceptual interest is the bijectivity of $h$, bestowed by the uniqueness of integral curves guaranteed by the Picard-Lindel{\"o}f theorem, which allows the application to generative models of probability densities~\cite{Rezende2015,Chen2018}.

A natural generalisation of neural differential equations in geometric deep learning is to consider ODEs on a smooth manifold $M$, rather than $\R^n$, to accommodate non-Euclidean data. Such manifold neural ODEs were introduced in the concurrent works~\cite{Falorsi2020,Lou2020,Mathieu2020}. The situation where the data exhibits symmetries under some group $G$ of transformations naturally entails the construction of equivariant NODEs which were obtained in~\cite{Kohler2020} for the Euclidean setting and more recently in~\cite{Katsman2021} for Riemannian manifolds $M$.

In this paper, we develop the mathematical foundations of equivariant manifold NODEs using their description in terms of differential geometry. In the past few years, the field of neural ODEs has seen several interesting developments in different directions such as novel model constructions, efficient implementations and competitive performance on applied problems. In~\cref{sec:related-work}, we provide an overview of the most relevant previous works and how they relate to our results. Our aim in the present paper is to complement these recent developments by providing a geometric framework for manifold neural ODEs, incorporating Lie theory of symmetries of differential equations. We then use the framework to analyse the theoretical modelling and approximation capabilities of manifold NODEs and to provide geometrical insight into their properties.

\subsection{Summary of results}
We first provide some mathematical background on differential geometry and Lie theory in~\cref{sec:mathematical_preliminaries}. In~\cref{sec:geom_framework} we then construct a manifestly geometric framework for NODEs on arbitrary smooth manifolds $M$ equivariant under the action of a connected Lie group $G$. In particular, we generalise previous constructions by removing assumptions on a metric structure on $M$ and establishing a stronger version of the fundamental relationship between the equivariance of the NODE model $h:M \to M$ and the corresponding ODE on $M$ (\cref{thm:main-equiv-NODE}). In addition, we provide a novel way of describing the NODE model space, based on the classical Lie theory of symmetries of differential equations, by parameterising equivariant vector fields $\phi$ on $M$ using differential invariants of the action of $G$ (\cref{thm:invariants-node}).

Using the geometric framework, we then investigate the approximation capabilities of manifold NODEs in 
\cref{sec:aug-eq-ndes}. In particular, we show how NODEs can be augmented in a way that respects the manifold structure of $M$ by embedding the diffeomorphism $h: M \to M$ in a flow on the tangent bundle $TM$ (\cref{thm:embed-flow}). We prove that the augmented NODE models are universal approximators when $M$ is connected (\cref{cor:node-univ-approx}). We proceed to show that the embedding into $TM$ is compatible with the induced action of $G$ on $TM$ (\cref{thm:embed-equiv-flow}), which proves that equivariant NODEs are universal approximators of equivariant diffeomorphisms (\cref{cor:equiv-node-univ-approx}). Furthermore, we show how the space of augmented equivariant NODEs can be parameterised in terms of higher-order differential invariants of $G$ (\cref{thm:invariants-aug-node}).

Finally, in~\cref{sec:eq-fields-and-den} we use the induced action of the diffeomorphism $h$ to construct NODE models for different types of fields on $M$. In particular, we consider explicitly equivariant scalars, densities (\cref{thm:induced-action-density}) and vector fields on $M$ (\cref{thm:induced-action-vector}).

The main contributions of the present paper are:

\begin{itemize}
    \item We develop a manifestly geometric framework for NODEs on an arbitrary smooth manifold $M$, equivariant with respect to the action of a connected Lie group $G$. In particular, we use the differential invariants of $G$ to parameterise the space of equivariant manifold NODEs.
    \item We show how NODEs can be augmented within our geometric framework using higher-order differential invariants in a way that respects both the manifold structure of $M$ and the equivariance under $G$. We prove that the resulting (equivariant) models are universal approximators for (equivariant) diffeomorphisms $h : M \to M$ of connected manifolds.
    \item We show how our framework can be used to model different kinds of equivariant densities and fields on $M$, using the induced action of the diffeomorphism $h$.
\end{itemize}

\subsection{Related work}\label{sec:related-work}
In this section, we provide an overview of important previous results in the literature related to equivariant manifold NODEs in general and our constructions and results in particular. In line with the focus of the present paper, we emphasise theoretical developments and contextualise our results where appropriate.

\paragraph{Equivariant continuous normalizing flows} Normalizing flows are generative models that construct complicated probability distributions by transforming a simple probability distribution (e.g., the normal distribution) using a sequence of invertible maps~\cite{Rezende2015,Chen2019}. Similar to residual networks, the sequence can be viewed as an Euler discretisation. In~\cite{Chen2018}, this connection was used to show that NODEs can model infinite sequences of infinitesimal normalizing flows, resulting in continuous normalizing flows (CNFs). This construction has been further developed in several directions, e.g., in~\cite{Grathwohl2019,Onken2021}.

Equivariant continuous normalizing flows on Euclidean space $\mathbb{R}^n$ were introduced in~\cite{Kohler2020}, extending previous constructions in~\cite{GarciaSatorras2021,Kohler2019,Rezende2019}. The prior and target densities in these normalizing flows share the same symmetries, provided that the maps are equivariant, which is equivalent to equivariance of the vector field generating the flow. 

\paragraph{Manifold neural ODEs} The mathematical framework required for extending NODEs to manifolds was developed in~\cite{Falorsi2020} which also provides a rigorous description of the adjoint method for backpropagation in terms of the canonical symplectic structure on the cotangent bundle. Explicit implementations of both forward and backward integration on Riemannian manifolds using dynamical trivialisations were developed in~\cite{Lou2020} based on the exponential map. The application considered for the manifold NODEs in these works is mainly restricted to continuous normalizing flows~\cite{Mathieu2020}. In this context~\cite{Rezende2020}, developing finite normalizing flows on tori and spheres, and~\cite{Katsman2023}, considering Riemannian residual networks, also deserve to be mentioned. 

Closest to the results presented in this paper is~\cite{Katsman2021} which considers flows on Riemannian manifolds equivariant under subgroups of isometries. In contrast, our work generalises the construction to a framework that accommodates arbitrary symmetry groups and manifolds and extends it from densities to more general fields on $M$ without requiring the existence of a metric.

\paragraph{Universality of neural ODEs} Since NODEs model invertible transformations as flows on $M$, which are restricted in that flow lines cannot intersect on $M$, not all diffeomorphisms $h:M \to M$ can be approximated using NODEs realised on the manifold $M$ itself. However, any diffeomorphism $h : M \to M$ can be expressed as a flow on some ambient space where $M$ is embedded~\cite{Utz1981}. This idea was first implemented in~\cite{Dupont2019} to define augmented NODEs for $M=\mathbb{R}^n$, and further developed in~\cite{Zhang2020} to prove that augmented NODEs on $\mathbb{R}^{2n}$ are universal approximators of diffeomorphisms $h : \mathbb{R}^n \to \mathbb{R}^n$. Using augmentation,~\cite{Bose2021} extended the universality to equivariant finite normalizing flows on $\mathbb{R}^n$. However, to the best of our knowledge, augmentation of manifold NODEs has not been previously considered in the literature. In this work, we use our geometric framework to construct augmented manifold NODEs and establish their universality for connected $M$. Furthermore, we generalise the construction to the equivariant case and show that universality persists in the presence of a non-trivial symmetry group $G$.

\paragraph{Differential invariants in neural networks}
An unresolved problem in the theory of manifold neural ODEs is to parameterise the space of vector fields, or at least a sufficiently large subset of them, to obtain expressive models~\cite{Falorsi2020}. In the equivariant setting, this has been accomplished by considering gradient flows of invariant potential functions on Euclidean spaces~\cite{Kohler2020} and Riemannian manifolds~\cite{Katsman2021}, respectively. This strategy, while practically appealing, restricts models to conservative vector fields.

Our approach to parameterising the neural ODE models is based on the theory of symmetries of differential equations (see, e.g., the seminal text~\cite{Olver1993}) by considering differential invariants of the action of the symmetry group $G$ on the manifold $M$. Recently, a similar approach was proposed in~\cite{Knibbeler2024} to address the parameterisation of invariant sections of homogeneous vector bundles related to equivariant convolutional neural networks~\cite{Kondor2018,Cohen2019,Aronsson2022}.

The theory for symmetries of differential equations was also used in~\cite{AkhoundSadegh2023} to create physics-informed networks (PINNs) where a loss term incorporating the infinitesimal generators of symmetry transformations is used to enforce approximate invariance of the equations under the action of the symmetry group. Even closer to our approach of using differential invariants is~\cite{Arora2024}, which constructs solutions to an ODE by first mapping to the space spanned by the differential invariants, learning a solution to the invariantised ODE using a PINN, and reconstructing the corresponding solution to the original equation. In both cases, PINNs are used to learn solutions to specific differential equations, in contrast to our approach where differential invariants are used to parameterise the space of equivariant differential equations and we attempt to learn the one which best approximates the diffeomorphism $h: M \to M$ mapping input data to output data. Somewhat conversely, equivariant CNNs were used in~\cite{Lagrave2022} to learn differential invariants of partial differential equations (PDEs) and use them to derive symmetry-preserving finite difference solution algorithms.

\paragraph{Flow matching}
An important recent contribution to the field of NODEs in the form of CNFs is the development of the flow matching (FM) framework that eliminates the need to perform the computationally expensive ODE simulations and gradient evaluations required in the maximum likelihood training approach originally proposed in~\cite{Chen2018}. The FM paradigm amounts to directly regressing the vector field generating the normalizing flow to the tangent vector of some desired path in the space of probability distributions. This approach was first introduced for the Euclidean case in~\cite{Lipman2023}, where the target probability path and corresponding vector fields were defined by marginalising over probability paths conditioned on the data samples to obtain conditional flow matching (CFM). The method was then further developed in~\cite{Tong2024} where mixtures of probability paths and vector fields conditioned on general latent variables were considered, and, in particular, optimal transport CFM (OT-CFM) was obtained by using the Wasserstein optimal transport map as the latent distribution.

The flow matching framework was extended to Riemannian manifolds in~\cite{Chen2024} at the expense of still requiring ODE simulations for manifolds without closed-form expressions for the geodesics. This result was then applied in~\cite{Yim2023,Bose2024} to construct equivariant models for protein structure generation by considering equivariant CNFs on the manifold $\mathrm{SE}(3)$. A general equivariant framework for flow matching in the Euclidean case was obtained in~\cite{Klein2023} by minimizing over the orbit of the symmetry group in the distance function used in OT-CFM, which amounts to aligning each sample pair along their respective orbits.

In the context of the work we present in this paper, it is important to emphasize that flow matching is a framework for training CNFs rather than a class of NODE models. However, the FM approach can be directly applied to CNF models, including augmented models, in our framework to provide scalable and efficient training based on the parameterisation of NODE model space in terms of differential invariants.

\section{Mathematical preliminaries}\label{sec:mathematical_preliminaries}

In this section, we review some mathematical concepts that are relevant to the remainder of the article. The reader is assumed to have a working knowledge of differential geometry and group theory.

\subsection{Group actions and flows on manifolds}
Let $G$ be a Lie group with Lie algebra $\mathfrak{g}$. The left action $\alpha:G\times M\to M$ of a Lie group $G$ on a manifold $M$ is an operator endowed with the identity and associativity property of the group structure of $G$. We introduce the operator $L_g:M\to M$ satisfying $L_g(p)=\alpha(g,p)$ for each $g\in G$ and $p\in M$. The \textit{orbit} of a point $p\in M$ is given by
\begin{equation*}
\mathcal{O}_p
    =\{L_g(p)\,:\,g\in G\}.
\end{equation*}
The action of $G$ is called \textit{semi-regular} if all orbits have the same dimension.

A \textit{vector field} $\phi$ is a smooth assignment of a tangent vector to each $p \in M$ and a \textit{flow} on $M$ is a map $\Phi : \mathbb{R} \times M \to M$ which is an action of the additive group $\mathbb{R}$, i.e., $\Phi(0,p)=p$ and $\Phi(s,\Phi(t,p))=\Phi(s+t,p)$ for all $p\in M$ and $s,t\in\mathbb{R}$. The flow $\Phi$ is generated by the vector field $\phi$ if for every point $p \in M$ we have
\begin{equation}\label{eq:gen-vec-field}
    \left. \frac{d}{dt} \Phi(t,p) \right|_{t=0} = \phi_p \,.
\end{equation}
Such a flow can be viewed as a solution of an ODE. Let $u:\mathbb{R}\to M$ be the integral curve solving the initial value problem
\begin{equation*}
    \frac{du}{dt} = \phi_{u(t)} \,, \quad u(0) = p \,.
\end{equation*}
It follows from $\Phi(0,p)=p$, equation~\eqref{eq:gen-vec-field} and the uniqueness of $u$, guaranteed by the Picard-Lindel{\"o}f theorem, that $u(t)=\Phi(t,p)$. Thus, the flow $\Phi$ generated by the vector field $\phi$ is the collection of all integral curves of $\phi$.

Given a vector field $\phi$, its corresponding flow is referred to as the \textit{exponentiation} of $\phi$
\begin{equation*}
    \Phi(t,p)=\exp{(t\phi)}(p).
\end{equation*}
It follows from the properties of a flow that $\exp$ satisfies $\exp{(0\phi)}p=p$, $\frac{d}{dt}\exp{(t\phi)p}=\phi_{\exp(t\phi)p}$ and $\exp{(t\phi)}\exp{(s\phi)}=\exp{((s+t)\phi)}$ for all $s,t\in\mathbb{R}$.

Given two flows $\sigma$ and $\psi$ generated by the vector fields $X$ and $Y$, respectively, the \textit{Lie derivative} $\mathcal{L}_XY$ is defined as the rate of change of $Y$ along $\sigma$. One can show that the Lie derivative of a vector field is given by the Lie bracket, $\mathcal{L}_XY=[X,Y]$.

The \textit{push-forward} (or \textit{differential}) of a smooth map $f:M\to N$ between smooth manifolds at a point $p\in M$ is a map $df_p:T_pM\to T_{f(p)}N$ such that, for all differentiable functions $h$ defined on a neighbourhood of $f(p)$,
\begin{equation*}
    (df_p(X_p))(h)=X_p(h\circ f),
\end{equation*}
where $X_p\in T_pM$. If $\gamma$ is a curve on $M$ such that $\gamma(0)=p$ and $\dot{\gamma}(0)=X_p$, this can also be written as
\begin{equation*}
    (df_p(X_p))(h)=\left.\frac{d}{dt}(h\circ f\circ\gamma(t))\right|_{t=0}.
\end{equation*}
Throughout this paper, we will primarily consider the push-forward of the left action $L_g:M\to M$ and denote this by $(L_g)_*:T_pM\to T_{L_gp}M$. To make the expression less cumbersome, we suppress the reference to $p$ in the notation $(L_g)_*$.

Let $f:M\to N$ be a map between two manifolds $M$ and $N$ each equipped with a left group action of $G$. Then $f$ is called \textit{equivariant} if $f\circ L_g={L}_g\circ f$ holds for all $g\in G$. A special case of equivariance is \textit{invariance}, i.e., $f\circ L_g=f$ for all $g\in G$. When determining the equivariance properties of manifold neural ODEs, we will primarily consider equivariance of vector fields, flows and diffeomorphisms. Vector fields and flows were defined above. A diffeomorphism is a smooth bijection $h : M \to M$ with a smooth inverse. Below, we define what equivariance means for each respective object.

\begin{definition}\label{dfn:flow_equivariance}
    A flow $\Phi:\mathbb{R}\times M\to M$ is  equivariant if $L_g\Phi(t,p)=\Phi(t,L_gp)$ for all $t\in\mathbb{R}$, $p\in M$ and $g\in G$.
\end{definition}

\begin{definition}\label{dfn:vector_field_equivariance}
    A vector field $\phi:M\to TM$ is  equivariant if $\phi_{L_gp}=(L_g)_*\phi_p$ for all $p\in M$ and $g\in G$.
\end{definition}

\begin{definition}\label{dfn:diffeomorphism_equivariance}
    A diffeomorphism $h:M\to M$ is equivariant if $L_g\circ h= g\circ L_g$ for all $g\in G$.
\end{definition}

\subsection{Symmetries of differential equations}\label{sec:diff_invariants}
We will now briefly review the general theory of Lie point symmetries of differential equations as originally envisioned by Lie and described geometrically in, e.g.,~\cite{Olver1993,Olver1995}. We consider a general ODE of the form
\begin{equation*}
    u^{(k)} = \phi\left(t,u,u^{(1)},\ldots,u^{(k-1)}\right) \,, \quad u^{k}(t) = \frac{d^k}{dt^k}u(t) \,,
\end{equation*}
or equivalently $\Delta\left(t,u,u^{(1)},\ldots,u^{(k)}\right)=u^{(k)}-\phi\left(t,u,u^{(1)},\ldots,u^{(k-1)}\right)=0$.

The solution $u(t)$ of an ODE is a curve on a manifold $M$ depending on a variable $t$. Let $T \cong \mathbb{R}$ be the space parameterised by the independent variable and $M$ the space of dependent variables. We define the \textit{total space} as the space $E=T \times M$. A smooth curve $u:\mathbb{R}\to M$ has \textit{graph} $\gamma_u=\{(t,u(t))\}\subset E$. The general theory allows for actions of $G$ on both the independent variables $T$ and the dependent variables $M$. In the application to NODEs, symmetries are considered only in terms of transformations of the inputs and outputs, and not as acting on the time $t$ parameterising the internal dynamics of the model. Thus, we restrict the left action of $G$ to only the dependent variables $M$, i.e., the left action on a graph $\{(t,u(t)\}\subset E$ is given by $\{(t,L_gu(t)\}$ for each $g\in G$.

The fundamental objects describing a Lie group $G$ of point transformations acting on the total space $E$, are the induced vector fields $X_i$, $i=1,\ldots,d = \mathrm{dim}\,G$ on $E$ defining a basis for the Lie algebra $\mathfrak{g}$. The induced vector fields generate the actions of $G$ via exponentiation, i.e., 
for each $g\in G$ there is a corresponding vector field $X_g$ such that $L_gp=(\exp{X_g})p$ for each $p\in M$. The group $G$ is a symmetry group of $\Delta$ if it preserves the solution space of $\Delta$, i.e., for every solution $u(t)$ and group element $g \in G$, the transformed function $L_g( u(t))$ is also a solution to $\Delta = 0$.

The geometric description uses the concept of jet bundles $J^{(k)}E$ which extends the space $E=T \times M$ to include the derivatives $u^{(1)},\ldots,u^{(k)}$. It is customary to write $J^{(k)}E=T\times M^{(k)}$, where $M^{(k)}$ is a manifold of dimension $k\dim{M}$ containing all the points in $M$ as well as their derivatives up to order $k$. The graph $\gamma_u$ of a smooth curve $u:\mathbb{R}\to M$ can be \textit{prolonged} to $J^{(k)}E$ as $\gamma_u^{(k)}=\{(t,u(t),\partial_tu(t),\dots,\partial_t^ku(t))\}\subset J^{(k)}E$. The \textit{prolonged group action} $L_g^{(k)}$ on the graph $\gamma_f^{(k)}$ is then defined as $L_g^{(k)}\gamma_f^{(k)}=\gamma_{L_gf}^{(k)}$. The induced action of $X$ on the derivatives is generated by the \textit{prolonged vector field} $X^{(k)}$, which can be used to express the symmetry condition succinctly in its infinitesimal form
\begin{equation*}
    \left. X^{(k)} \left( u^{(k)} - \phi\left(t,u,u^{(1)},\ldots,u^{(k-1)}\right) \right) \right|_{\Delta=0} = 0\,. 
\end{equation*}
Note that the restriction of $L_g$ to $M$ implies that $L_g^{(k)}$ and the corresponding generating vector fields are restricted to $M^{(k)}$.

Elementary results from the theory of algebraic equations on manifolds, applied to the manifold $J^{(k)}E$, can then be used to construct the most general ODE which is equivariant under the symmetry group $G$. The construction is based on \textit{differential invariants}, i.e., functions $I : J^{(k)}E \to \R$ which are invariant under the action of $G$, which infinitesimally means $X_i^{(k)}(I)=0$ for $i=1,\ldots,d$. Since a function of a set of invariants is also an invariant, we only need to consider the complete set of functionally independent invariants. The number of such invariants is given by the following \namecref{lemma:invariants-dim}.

\begin{theorem}[Theorem~2.34~\cite{Olver1995}]
\label{lemma:invariants-dim}
    If the prolonged group $G^{(k)}$ of a group $G$ acts semi-regularly on $J^{(k)}E$ with orbit dimension $s_k$, there are $\dim{J^{(k)}E}-s_k$ functionally independent local differential invariants of order $k$.
\end{theorem}

Having a complete set of functionally independent differential invariants makes it possible to express the most general $G$-equivariant ODE.

\begin{theorem}[Proposition~2.56~\cite{Olver1993}] \label{thm:general-sym-ODE}
Let $G$ be a Lie group acting semi-regularly on $E = T \times M$ and let $I_1,\ldots,I_{\mu_k}$ be a complete set of functionally independent differential invariants of order $k$. Then, the most general ODE of order $k$ for which $G$ is a symmetry group is locally on the form
\begin{equation*}
    F(I_1,\ldots,I_{\mu_k}) = 0
\end{equation*}
for an arbitrary function $F : \mathbb{R}^{\mu_k} \to \mathbb{R}^n$, where $n$ is the dimension of $M$.
\end{theorem}

\begin{remark}
In the above construction, we focus exclusively on the case where the space of independent variables, $T$, is isomorphic to $\mathbb{R}$. The general setting -- which includes the theory for partial differential equations (PDEs) -- allows $T$ to be of arbitrary dimension. However, since our framework is concerned with ODEs, including such a general setting would be redundant and require cumbersome notation. See, e.g.,~\cite{Olver1993,Olver1995} for a more comprehensive overview.
\end{remark}

\section{Geometric framework}\label{sec:geom_framework}
In this section, our aim is to establish the appropriate geometric framework for describing manifold NODEs on a smooth manifold $M$ of dimension $\mathrm{dim} \, M = n$, and in particular their equivariance under a group $G$ of symmetry transformations acting on $M$. Throughout, we will assume that $M$ is connected and that $G$ is a connected Lie group whose action on $M$ is semi-regular.

In~\cref{sec:manifold-neural-odes}, we give a precise definition of NODEs and their equivariance. In~\cref{sec:eq-on-nodes}, we generalise previous results on the equivalence between an NODE, its generating vector field and the corresponding flow on $M$. Finally,~\cref{sec:diff-inv-nodes} introduces the notion of differential invariants in the context of NODEs and how they can be used to parametrise an NODE based on~\cref{thm:general-sym-ODE}.

\subsection{Manifold neural ODEs}\label{sec:manifold-neural-odes}
Manifold neural ODEs are neural network models defined by a vector field describing how the data evolves continuously over time governed by an ordinary differential equation. The word \textit{manifold} emphasises that the data can be defined on non-Euclidean manifolds. We define manifold neural ODEs in the following way:

\begin{definition}[Manifold neural ODE]\label{dfn:manifold_neural_odes}
Given a point $p\in M$ and a learnable vector field $\phi:M\to TM$, let $u:\mathbb{R}\to M$ be the unique curve solving the Cauchy problem
\begin{equation*}
    \dot{u}(t) = \phi_{u(t)} \,, \quad u(0) = p \,.
\end{equation*}
A \textit{manifold neural ODE} on $M$ is the diffeomorphism $h:M\to M$ defined by $h:u(0)\mapsto u(1)$. The point $p\in M$ is referred to as the input and $h(p)\in M$ as the output.
\end{definition}
\begin{remark}
    The map $h:M\to M$ can, equivalently, be defined by $h:\Phi(0,p)\mapsto\Phi(1,p)$, where $\Phi$ is the flow on $M$ generated by a learnable vector field $\phi:M\to TM$.
\end{remark}

A neural network is said to be equivariant with respect to a group $G$ if applying the group action to the input before applying the network gives the same result as applying the group action to the output after applying the network.~\Cref{dfn:manifold_neural_odes} allows us to define equivariance of a manifold neural ODE in a concise way:

\begin{definition}[Equivariant manifold neural ODE]
   A manifold neural ODE $h:M\to M$ is said to be $G$-equivariant if $h\circ L_g=L_g\circ h$ holds for all $g\in G$.
\end{definition}

\subsection{Equivariance of neural ODEs}\label{sec:eq-on-nodes}
In this section, we introduce~\cref{thm:main-equiv-NODE} which demonstrates the equivalence between equivariance of a vector field $\phi$, its generated flow $\Phi$ and the corresponding diffeomorphism $h$ mapping the input to the output. The implication (i) to (iii) in~\cref{thm:main-equiv-NODE} below was proven for $M=\R^n$ in~\cite{Kohler2020} and the equivalence between (i) and (ii) was established for Riemannian manifolds in~\cite{Katsman2021}. Our result further extends these connections by establishing the equivalence between equivariance of the vector field defining the Cauchy problem, the flow that it generates, and the diffeomorphism $h: M \to M$ defining the corresponding NODE. Furthermore, our construction works for any connected manifold without the requirement of a metric structure on $M$.
The proof in~\cite{Katsman2021}, in fact, makes no explicit use of the Riemannian metric and establishes the equivalence of (i) and (ii) also in our more general setting. We include this as part of the proof for completeness.

\begin{theorem}\label{thm:main-equiv-NODE}
Let $\phi$ be a smooth vector field on the manifold $M$, $\Phi : \mathbb{R} \times M \to M$ be the flow generated by $\phi$ and $h : M \to M$ be the diffeomorphism defined by $p=u(0)$, $h(p) = u(1)$, where $u(t)$, $t \in [0,1]$, is the unique solution to the Cauchy problem
\begin{equation}
\label{eqn:manifold-NODE}
    \dot{u}(t) = \phi_{u(t)} \,, \quad u(0) = p \,.
\end{equation}
Then the following are equivalent:
\begin{itemize}
    \item[(i)] $\phi$ is a $G$-equivariant vector field,
    \item[(ii)] $\Phi$ is a $G$-equivariant flow,
    \item[(iii)] $h$ is a $G$-equivariant diffeomorphism.
\end{itemize}
\end{theorem}
\begin{proof}

We begin by establishing the equivalence between (i) and (ii). Assume that $\phi$ is $G$-equivariant. We want to show that $\Phi(t,L_g p)=L_g\Phi(t,p)$ for all $g \in G$, $p\in M$. By the properties of flows, we have $\Phi(0,L_g p)=L_g p=L_g\Phi(0,p)$. The equivariance of $\phi$ (\cref{dfn:vector_field_equivariance}) then implies
\begin{equation}
\label{eq:equiv_nde}
    \left.\frac{d}{dt}\Phi(t,L_g p)\right|_{t=0} = \phi_{L_g p} = (L_g)_*\phi_{p}\\ = \left.\frac{d}{dt}L_g \Phi(t,p)\right|_{t=0}
\end{equation}
for any point $p\in M$. Consequently, the tangent of $\Phi(t,L_g p)$ coincides with the tangent of $L_g\Phi(t,p)$ at every point in $M$, and the uniqueness of the integral curve then shows that $\Phi$ is $G$-equivariant. Conversely, if $\Phi$ is $G$-equivariant, a rearrangement of Equation~\eqref{eq:equiv_nde} shows that $\phi$ is $G$-equivariant. Thus, (i) and (ii) are equivalent.

We complete the proof by showing that (ii) and (iii) are equivalent. To show that (ii) implies (iii) is trivial, since it follows directly from the fact that $h(p)=\Phi(1,p)$. To show the converse, however, requires some more work.

Suppose that the diffeomorphism $h$ is $G$-equivariant. We know that the flow generated by $\phi$ is given by the exponentiation of $\phi$, i.e., $\Phi(t,p)=\exp{(t\phi)}p$. It follows that $h$ can be expressed in terms of the vector field $\phi$ as
\begin{equation*}
    h(p)=\exp(\phi)p.
\end{equation*}
Since the group $G$ is a Lie group, there is a vector field $X_g$ on $M$ generating the action of the group element $g\in G$ as $L_g p=\exp{(X_g)}p$. The map $\Psi:\mathbb{R}\times M\to M$ given by
\begin{equation*}
    \Psi(s,p)=\exp(sX_g)p
\end{equation*}
defines a flow on $M$. Equivariance of $h$, which can be expressed as $h \circ L_g = L_g \circ h$, implies that
\begin{equation*}
    \exp(\phi)\exp{(X_g)}=\exp(X_g)\exp{(\phi)}.
\end{equation*}
This is true if and only if the Lie bracket vanishes $\left[\phi,X_g\right]=0$, or equivalently, the Lie derivative of $\phi$ along the flow generated by $X_g$ does. A well-known result (see e.g.~\cite{Nakahara2003}) is that this holds if and only if $\Psi(s,\Phi(t,p))=\Phi(t,\Psi(s,p))$, from which it follows that $L_g\Phi(t,p)=\Phi(t,L_g p)$. This completes the proof.
\end{proof}

\Cref{thm:main-equiv-NODE} implies that the neural ODE defined by a $G$-equivariant vector field has a $G$-equivariant solution. An alternative formulation of this result, obtained directly from the connection between the flow $\Phi$ and the solution $u(t)$ to~\eqref{eqn:manifold-NODE}, is that $L_g u(t)$ is also a solution to~\eqref{eqn:manifold-NODE} for every $g \in G$. This is the definition of equivariance occurring in the classical theory of (continuous) symmetries of differential equations and naturally extends the notion of equivariance layer by layer in an equivariant neural network to the continuum limit since each point on the solution $u(t)$ is transformed by the action of the same group element $g \in G$.

\subsection{Differential invariants and equivariant NODEs}\label{sec:diff-inv-nodes}
In~\cref{sec:diff_invariants} we introduced the notion of differential invariants and saw how they can be useful when determining the most general ODE which is equivariant with respect to some group $G$. We will now show how this theory can be used to provide the most general parametrisation of a manifold NODE with symmetry group $G$. Combining~\cref{thm:main-equiv-NODE} and~\cref{thm:general-sym-ODE} entails the following theorem.

\begin{theorem}\label{thm:invariants-node}
Given a complete set of first-order differential invariants $I_1,\ldots,I_{\mu_1}$ of $G$, the most general $G$-equivariant NODE is (locally) on the form $H(I_1,\ldots,I_{\mu_1})=0$ where $H:\mathbb{R}^{\mu_1} \to \mathbb{R}^n$ is an arbitrary function.
\end{theorem}

In practice, it is often convenient to rewrite $H(I_1,\ldots,I_{\mu_1})=0$ on the equivalent form~\eqref{eqn:manifold-NODE} to obtain the most general form of the equivariant vector field. In this way the differential invariants of $G$ can be used to parameterise the equivariant vector fields on $M$, generalising the construction in~\cite{Katsman2021} and incorporating it into the geometric framework. We emphasise that there is no additional structural conditions on the function; equivariance is obtained through the use of the differential invariants.

To demonstrate the framework based on differential invariants we will now consider its application to two examples where $M$ is, respectively, Euclidean and non-Euclidean.
\begin{example}\label{example:rot_eq_NODE}
Let $M = \mathbb{R}^2\setminus\left\{0\right\}$ and $G = \SO(2)$, and let $G$ act on $M$ by rotations in the usual way, i.e.,
\begin{equation*}
    L_\epsilon [x,y]=[x\,\cos{\epsilon}-y\,\sin{\epsilon},x\,\sin{\epsilon}+y\,\cos{\epsilon}],
\end{equation*}
where $u(t) = [x(t),y(t)]$ and $\epsilon$ is the arbitrary angle of rotation. The fix-point at the origin in $\mathbb{R}^2$ is excluded to obtain a semi-regular action of $\SO(2)$. The induced vector field $X$ generating $G$ is given by $X=-y\partial_x+x\partial_y$, and its first-order prolongation $X^{(1)}$ given by 
\begin{equation*}
    X^{(1)}=-y\partial_x+x\partial_y-\dot{y}\partial_{\dot{x}}+\dot{x}\partial_{\dot{y}}.
\end{equation*}
By~\cref{lemma:invariants-dim}, there are four functionally independent first-order differential invariants, $I:J^{(1)}E\rightarrow\R$, satisfying $X^{(1)}(I)=0$. By solving the corresponding characteristic system, they are found to be
\begin{gather*}
    I_1=t \,,\quad I_2 = r = \sqrt{x^2+y^2}\,,\\ I_3=r\dot{r}=x\dot{x}+y\dot{y} \,,\quad  I_4=r^2\dot{\theta} = x\dot{y}-y\dot{x},
\end{gather*}
where $(r,\theta)$ are polar coordinates on $M$. By~\cref{thm:general-sym-ODE}, the most general first-order ODE for which $G$ is a symmetry group is of the form $H(t,r,r\dot{r},r^2\dot{\theta})=0\in\mathbb{R}^2$, which is equivalent to 
\begin{equation}\label{eq:rot_eq}
\begin{cases}
    \dot{x}=A(t,r)x-B(t,r)y\\
    \dot{y}=B(t,r)x+A(t,r)y
\end{cases}
,
\end{equation}
where $A$ and $B$ are two arbitrary functions. Thus, rotationally equivariant diffeomorphisms of $M$ are obtained from the NODE model by parameterising the space of equivariant vector fields in terms of the functions $A$ and $B$ in~\eqref{eq:rot_eq} and using neural networks to learn approximations of $A$ and $B$. This generalises the constructions appearing in~\cite{Chen2018,Lou2020,Katsman2021}. Again, we emphasise that the neural networks are not required to be equivariant with respect to the symmetry group $G$; equivariance of the NODE model is obtained through the use of differential invariants to derive the system of ODEs in~\eqref{eq:rot_eq}.

To illustrate how the above construction can be used in practice, we consider a rotationally equivariant diffeomorphism on the plane $\mathbb{R}^2$ that can be learned using an NODE of the form~\eqref{eq:rot_eq}. Let $h:\mathbb{R}^2\to \mathbb{R}^2$ be a unit translation in the radial direction, i.e.,
\begin{equation*}
    h(x,y)=\left(x+\frac{x}{r},y+\frac{y}{r}\right)\,.
\end{equation*}
 for $x,y\in\mathbb{R}$ and $r=\sqrt{x^2+y^2}$. Two functions $x$ and $y$ with the property $(x(1),y(1))=h(x(0),y(0))$ are $x(t)=x(0)+t\,{x(0)}/{r}$ and $y(t)=y(0)+t\,{y(0)}/{r}$. They have derivatives $\dot{x}(t)={x(0)}/{r}$ and $\dot{y}(t)={y(0)}/{r}$, respectively. Consequently, in~\eqref{eq:rot_eq} we have $A(t,r)=1/r$ and $B(t,r)=0$.

In this example, we let the data set consist of $64$ equally spaced points on $[-2,2]\times[-2,2]$. The functions $A$ and $B$ are modelled using feed-forward neural networks and~\eqref{eq:rot_eq} is defined and trained using the {\tt torchdiffeq} package~\cite{Chen2018}. The functions $A$ and $B$ in~\eqref{eq:rot_eq} are plotted as functions of $r$ in~\cref{fig:example1} after 300 epochs of training. From the definition of $h$, we expect $A(r)$ to approach a function $a(r)=1/r$ and $B(r)$ to approach a function $b(r)=0$ during training, which can be observed to be the case in the~\cref{fig:example1}.

\begin{figure}[!ht]
\centering
\includegraphics[scale=1]{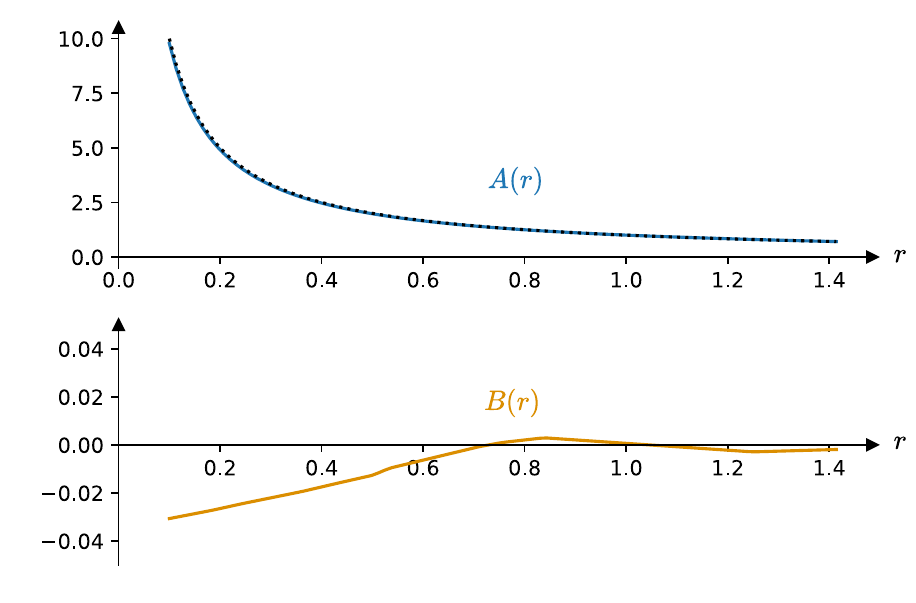}
\caption{The figure shows $A(r)$ (blue line) and $B(r)$ (orange line) in~\cref{example:rot_eq_NODE} as functions of $r$ after $300$ epochs. The dashed line in the top figure illustrates the function $a(r)=1/r$.}
\label{fig:example1}
\end{figure}

\end{example}

\begin{remark}
    In our framework, $I = t$ is always an invariant since we consider the case where $G$ acts trivially on $T$. In our examples, we consider autonomous NODE systems, in which case the functions $A$ and $B$ are independent of $t$.
\end{remark}
\begin{remark}
    In the example above, rotational equivariance is conveniently expressed in terms of the polar coordinates $(r,\theta)$. For any one-dimensional symmetry group $G$ we can similarly find canonical coordinates on $M$ defined by $G$ acting by translation in one coordinate and trivially in all others. For higher-dimensional symmetry groups, however, this is not generally possible.   
\end{remark}

\begin{example}\label{example:so2_sphere}
We now consider a non-Euclidean example where $G=\SO(2)$ acts on the sphere $M = S^2$ through rotation around the $z$-axis. If $\theta$ is the polar angle and $\varphi$ is the azimuthal angle, the action of $\SO(2)$ can, in spherical coordinates, be given by
\begin{equation*}
    L_{\epsilon} [\theta,\varphi]=[\theta,\varphi+\epsilon],
\end{equation*}
where $\epsilon$ is the rotation angle. The action is generated by the vector field $X$ and its prolongation $X^{(1)}$ given by
\begin{equation*}
    X^{(1)}=X=\partial_\varphi.
\end{equation*}
Since $J^{(1)}S^2$ is five-dimensional and $\SO(2)$ has dimension one, there are four first-order functionally independent differential invariants according to~\cref{lemma:invariants-dim} which are readily computed as
    \begin{equation*}
    I_1=t,\ \ I_2=\theta,\ \ I_3=\dot{\theta}, \ \ \mathrm{and} \ \ I_4 = \dot{\varphi}.
\end{equation*}
~\Cref{thm:invariants-node} implies that the most general ODE is given by
\begin{equation}\label{eq:ex_sphere}
    \begin{cases}
        \dot{\theta}&=A(t,\theta)\\
        \dot{\varphi}&=B(t,\theta)
    \end{cases},
\end{equation}
where $A$ and $B$ are arbitrary functions.

A rotationally equivariant map on $S^2$ can be obtained through scaling of the polar angle and shifting of the azimuthal angle, i.e.,
\begin{equation*}
    h:(\theta,\varphi)\mapsto(\theta e^\epsilon,\varphi+\nu)
\end{equation*}
for some real numbers $\epsilon$ and $\nu$. Two functions $\theta$ and $\varphi$ that satisfy $(\theta(1),\varphi(1))=h(\theta(0),\varphi(0))$ are $\theta(t)=\theta(0)\,e^{t{\epsilon}}$ and $\varphi(t)=\varphi(0)+t\,\nu$. In~\eqref{eq:ex_sphere}, this corresponds to $A(t,\theta)=\theta\epsilon$ and $B(t,\theta)=\nu$. 

In this example, we consider $\epsilon=1$ and $\nu=0.05$, giving the derivatives $\dot{\theta}(t)=\theta$ and $\dot{\varphi}(t)=0.05$. We define an NODE of the form~\eqref{eq:ex_sphere} and train it for a set of 400 randomly distributed points on the sphere being transformed by the map $h$. The training is performed in the same way as in~\cref{example:rot_eq_NODE}. In~\cref{fig:example-on-sphere}, we have plotted the resulting $A$ and $B$ as functions of $\theta$ after 400 epochs of training. We see that $A(\theta)$ is approximately equal to a function $a(\theta)=\theta$ and $B(\theta)$ is approximately equal to a function $b(\theta)=0.05$, as expected.

\begin{figure}[!ht]
\centering
\includegraphics[scale=1]{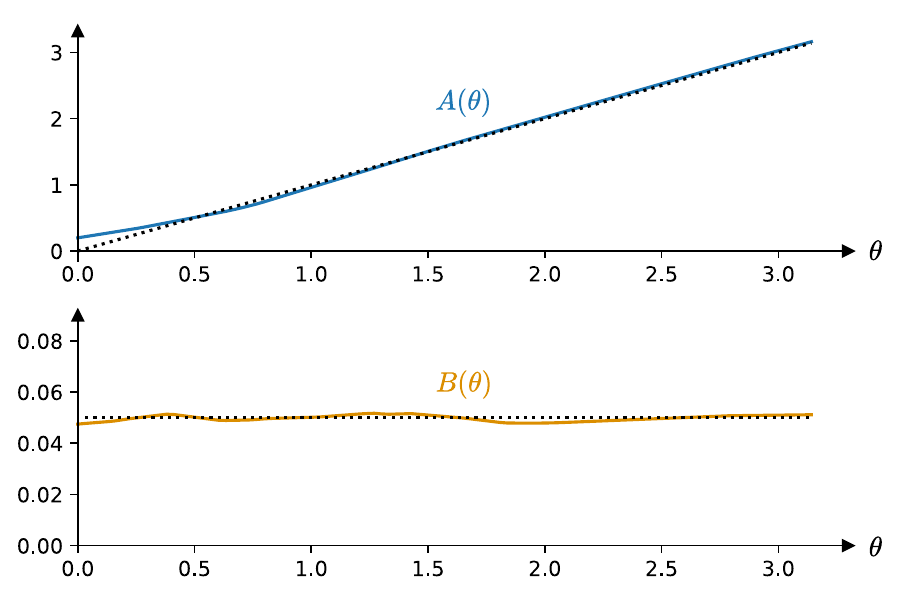}
\caption{$A$ and $B$ in~\cref{example:so2_sphere} as functions of $\theta$ after 400 epochs of training. The dashed lines represent the functions $a(\theta)=\theta$ and $b(\theta)=0.05$, respectively.}
\label{fig:example-on-sphere}
\end{figure}
\end{example}

We have now seen two examples where equivariant manifold NODEs can be trained to approximate equivariant diffeomorphisms $h : M \to M$. Next, we turn to the question of whether this is possible for any diffeomorphism on any manifold.

\section{Augmentation and universality}\label{sec:aug-eq-ndes}
A well-known issue with neural ODEs is their inability to learn certain classes of diffeomorphisms. Because the trajectories of the ODE solutions can not intersect, not all diffeomorphisms $h:M \to M$ can be obtained from integral curves on $M$. The canonical example in the Euclidean setting $M=\mathbb{R}^n$ is the diffeomorphism $h(x) = -x$, which cannot be represented by an NODE. The resolution of this issue for the case $M=\mathbb{R}^n$ was discussed in~\cite{Dupont2019}, where augmented neural ODEs were introduced, and in~\cite{Zhang2020}, where the authors prove that NODEs on the augmented space $\mathbb{R}^{2n}$ are universal approximators of diffeomorphisms $h : \mathbb{R}^n \to \mathbb{R}^n$. Heuristically, augmenting the state space of the ODE amounts to introducing enough extra dimensions to resolve intersections of the integral curves. 

Augmentation of NODEs is equivalent to the problem of embedding the diffeomorphisms $h : M \to M$ in a flow, i.e., the integral curve of an ODE on some ambient space. In~\cite{Utz1981} it was shown that any $h$ can be embedded in a flow on an ambient twisted cylinder of dimension $n+1$ obtained as a fibration over $M$. However, the properties of the ambient space, e.g., the topological class of the fibration, depend on the diffeomorphism. As a consequence, in the setting where the objective is learning $h$, we must instead consider an embedding into an ambient space that is common for all $h$.

We first consider augmentation of manifold NODEs for the non-equivariant case, corresponding to a trivial symmetry group $G$, to establish universality of manifold NODEs on connected $M$. We then show that the construction is equivariant with respect to the action of $G$ which implies that equivariant manifold NODEs are universal approximators of equivariant diffeomorphisms $h:M \to M$. The geometric perspective offered by our framework is essential to the construction and its subsequent generalisation to include non-trivial symmetry groups $G$.

\subsection{Augmented manifold NODEs}
A natural candidate for an ambient space that allows us to resolve intersections of integral curves in the case of manifold NODEs is the tangent bundle $TM$ of the manifold $M$. The following theorem shows that it is possible to embed any diffeomorphism on $M$ into a flow on the tangent bundle if $M$ is connected, and provides the theoretical basis for augmentation of manifold NODEs. 

\begin{theorem}\label{thm:embed-flow}
Let $M$ be a connected, smooth manifold and $h : M \to M$ be a diffeomorphism. Then $h$ can be embedded in a flow on $TM$.
\end{theorem}
\begin{proof}
Because $M$ connected it is also path-connected, so for every $p \in M$ there exists a path $\gamma_p : [0,1] \to M$ such that $\gamma_p(0)=p$, $\gamma_p(1) = h(p)$. Let $\Gamma_p$ be the lift of $\gamma_p$ to $TM$, defined by
\begin{equation}
\label{eq:gamma-lift}
    \Gamma_p(t) = \left[ \gamma_p(t), \frac{d}{dt} \gamma_p(t) \right] \,.
\end{equation}
The obstruction to embedding $\gamma_p$ into a flow on $M$ are intersections between the paths $\gamma_p$, i.e., points where $\gamma_p(\tau)=\gamma_{p'}(\tau')$ for some $\tau,\tau' \in [0,1]$. The paths $\gamma_p$ and $\gamma_{p'}$ can always be continuously deformed to intersect only in isolated points $\tau,\tau' \in [0,1]$ so that
\begin{equation*}
    \frac{d}{dt}\left.\gamma_p(t)\right|_{t=\tau} \neq \frac{d}{dt}\left.\gamma_{p'}(t)\right|_{t=\tau'}. \,
\end{equation*}
The corresponding lifts $\Gamma_p$ and $\Gamma_{p'}$, by construction, never intersect in $TM$ even where $\gamma_p$ and $\gamma_{p'}$ intersect in $M$. Consequently, the gradients of $\Gamma_p$ define a vector field on $TM$ and a corresponding flow $\Phi : \mathbb{R} \times TM \to TM$ into which $h$ is embedded as
\begin{equation}
\label{eq:h-embedding}
    h(p) = \pi \Phi\left(1,p,\left.\frac{d}{dt}\gamma_p(t)\right|_{t=0}\right) \,,
\end{equation}
where $\pi:TM \to M$ is the tangent bundle projection.
\end{proof}

\cref{thm:embed-flow} shows that it is possible to augment manifold NODEs by embedding the diffeomorphism $h$ into a flow on $TM$. The augmented NODE model corresponding to this embedding is the Cauchy problem describing the evolution of a state $U : \mathbb{R} \to TM, \, U(t) = \left[u(t),\dot{u}(t)\right]$, generated by a vector field $\phi = [\chi,\psi]$ on $TM$ according to
\begin{equation}
\label{eq:augmented-NODE}
    \dot{U}(t) = \phi_{U(t)} \,, \quad U(0) = [u(0), \dot{u}(0)]\,,
\end{equation}
where we require that $\chi(u,\dot{u}) = \dot{u}$, so that $\phi(u,\dot{u}) = [\dot{u},\psi(u,\dot{u})]$, in order for $\phi$ to be consistent with tangent vectors to a lift $\dot{U}(t) = [\dot{u}(t),\ddot{u}(t)]$. The diffeomorphism $h: M \to M$ described by the augmented NODE model is obtained by $p=u(0)$, $h(p)=u(1)=\pi U(1)$, and we may take $\dot{u}(0)=0$ as a convenient initial condition for $\dot{u}$. By~\cref{thm:embed-flow}, the resulting augmented manifold NODE can approximate any diffeomorphism on $M$.
\begin{corollary}\label{cor:node-univ-approx}
Augmented manifold NODEs are universal approximators of diffeomorphisms $h : M \to M$.
\end{corollary}

This result generalises the construction in~\cite{Zhang2020} to arbitrary manifolds $M$ and provides an explanation of the geometric origin of the augmentation. The idea behind the augmented construction is illustrated in~\cref{fig:intersection}.

\begin{figure}[!ht]
\centering
\includegraphics[scale=0.4]{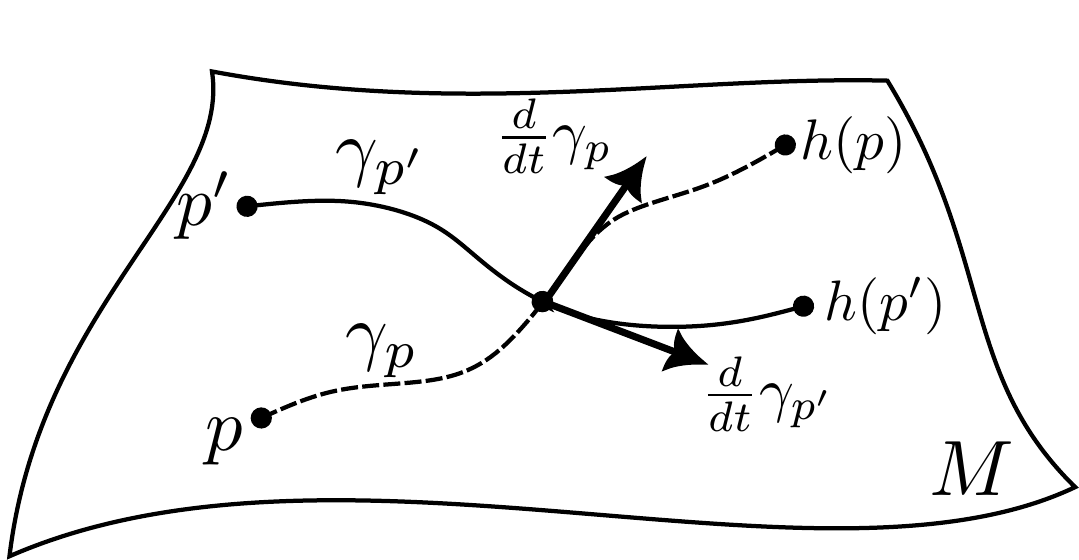}
\caption{Intersection of paths $\gamma_p$ and $\gamma_{p'}$.}
\label{fig:intersection}
\end{figure}

We note that the augmented manifold NODE is equivalent to the second-order ODE $\ddot{u}(t) = \psi(u,\dot{u})$ on $M$, a fact that will prove useful when we now proceed to consider augmenting equivariant manifold NODEs.

\begin{remark}
The lift $\Gamma_p$ can alternatively be described in terms of a section of the first jet bundle $J^{(1)}$. Specifically, the condition $\chi(u,\dot{u}) = \dot{u}$ corresponds to restricting the vector field $\phi$ to $J^{(1)}$. However, for the purpose of describing flows on $M$, and the subsequent construction of augmented NODEs, the formulation in terms of the tangent bundle $TM$ and the lift~\eqref{eq:gamma-lift} is more convenient.
\end{remark}

\begin{remark}
Connectedness of $M$ is a strict limitation in any flow-based model, since integral curves can never span disconnected components. An equivalent statement to the one made above is that manifold NODEs are universal approximators for diffeomorphisms restricted to the individual connected components of $M$.
\end{remark}

\subsection{Augmented equivariant manifold NODEs}
In order to extend augmentation to equivariant manifold NODEs we first show  that the construction in~\cref{thm:embed-flow} is equivariant with respect to the action of $G$ on $TM$ induced by the action on $M$. 

\begin{theorem}\label{thm:embed-equiv-flow}
Let $G$ be a Lie group acting on a connected, smooth manifold $M$ and let $h:M \to M$ be an equivariant diffeomorphism satisfying $h(L_g p) = L_g h(p)$ for every $p \in M$ and $g \in G$. Then $h$ can be embedded in a $G$-equivariant flow on $TM$.
\end{theorem}
\begin{proof}
Let $p \in M$, $\gamma_p$ be a path from $p$ to $h(p)$ and $\Gamma_p$ be the lift~\eqref{eq:gamma-lift}. For every $L_g p$ in the orbit of $p$ we construct the corresponding path as $\gamma_{L_g p} = L_g \gamma_p$, which satisfies $\gamma_{L_g p}(0) = L_g p$, $\gamma_{L_g p}(1) = L_g h(p) = h(L_g p)$ as required, by the equivariance of $h$. The action of $G$ on $M$ induces an action on the lift in $TM$ as
\begin{equation}
\label{eq:action-gamma-lift}
    L_g \Gamma_p(t) = \left[ L_g \gamma_p(t), (L_g)_* \frac{d}{dt} \gamma_p(t) \right] \,.
\end{equation}
From the definition of $\gamma_{L_g p}$ we have
\begin{equation*}
    (L_g)_* \frac{d}{dt} \gamma_p(t) = \frac{d}{dt} L_g \gamma_{p}(t) = \frac{d}{dt} \gamma_{L_g p}(t) \,,
\end{equation*}
which implies that $L_g \Gamma_p = \Gamma_{L_g p}$.

The vector tangent to the lift $\Gamma_p$ in $TM$ is
\begin{equation*}
    \frac{d}{dt} \Gamma_p(t) = \left[ \frac{d}{dt} \gamma_p, \frac{d^2}{dt^2} \gamma_p, \right] \,.
\end{equation*}
Considering the induced action of $G$ in~\eqref{eq:action-gamma-lift} and its push-forward, we have
\begin{equation}
    (L_g)_* \frac{d}{dt}\Gamma_p = \left[\frac{d}{dt} \gamma_{L_g p} , \frac{d^2}{dt^2} \gamma_{L_g p} \right] =
    \frac{d}{dt}\Gamma_{L_g p} \,.
\end{equation}
Consequently, the flow $\Phi : \mathbb{R} \times TM \to TM$ generated by the gradients of $\Gamma_p$ is equivariant, or equivalently, compatible with the induced action of $G$ on the lift $\Gamma_p$.

Finally, since $\pi \circ L_g = L_g \circ \pi$ by construction, the embedding of $h$ in~\eqref{eq:h-embedding} preserves the equivariance
\begin{align*}
    h(L_g p) &= \pi \Phi\left(1,L_g p,\frac{d}{dt}\left.\gamma_{L_g p}(t)\right|_{t=0}\right) \\
    &= \pi L_g \Phi\left(1,p,\frac{d}{dt}\left.\gamma_{p}(t)\right|_{t=0}\right) =  L_g h(p)\,.
\end{align*}
\end{proof}

The lift $\Gamma_p$ and its tangent vector are illustrated in~\cref{fig:lift-gamma-p}.
\begin{figure}[!ht]
\centering
\includegraphics[scale=0.43]{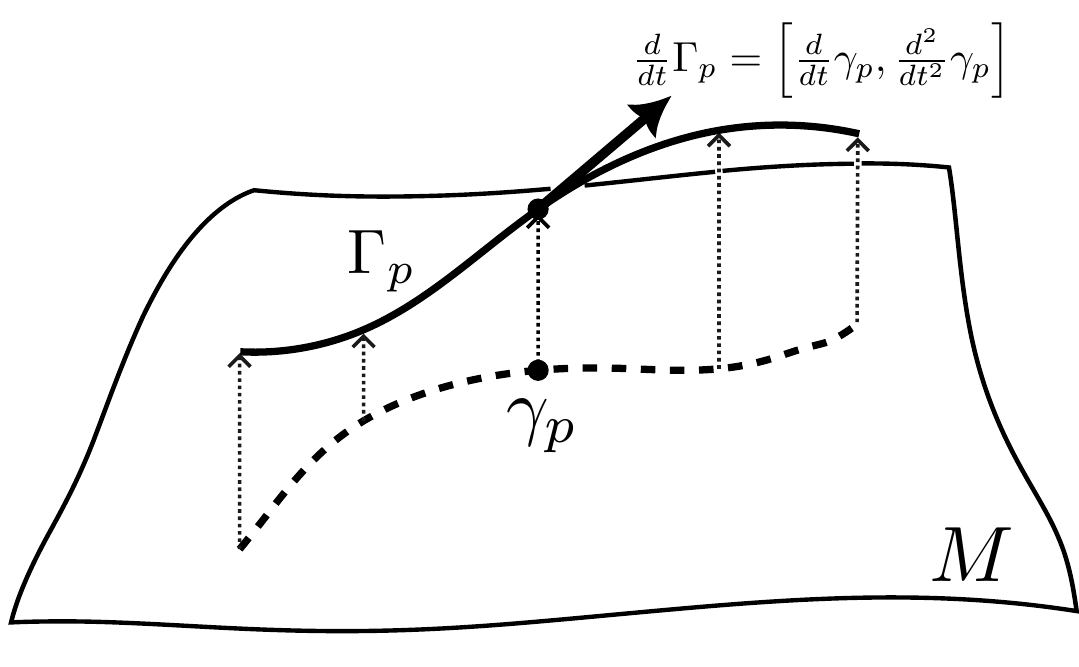}
\caption{The lift $\Gamma_p$ of $\gamma_p$ and the vector $\frac{d}{dt}\Gamma_p$ tangent to $TM$.}
\label{fig:lift-gamma-p}
\end{figure}

\Cref{thm:embed-equiv-flow} guarantees that any equivariant diffeomorphism $h:M \to M$ can be embedded in an equivariant flow on $TM$, and that augmentation can be extended to equivariant manifold NODEs. To incorporate equivariance under the induced action of $G$ on $TM$ into the augmented manifold NODE in~\eqref{eq:augmented-NODE}, we require that the vector field $\phi$ is equivariant.~\Cref{thm:main-equiv-NODE} then ensures that the solution to~\eqref{eq:augmented-NODE} is an equivariant flow $\Phi : \mathbb{R} \times TM \to TM$ which can approximate any equivariant diffeomorphism $h : M \to M$ by the embedding in~\cref{thm:embed-equiv-flow}.

\begin{corollary}\label{cor:equiv-node-univ-approx}
Augmented $G$-equivariant manifold NODEs are universal approximators of $G$-equivariant diffeomorphisms $h : M \to M$.
\end{corollary}

From the observation that the augmented manifold NODE is equivalent to the second-order ODE $\ddot{u}(t) = \psi(u,\dot{u})$ on $M$, it follows that differential invariants of $G$ can be used to parameterise the equivariant vector fields $\phi = [\dot{u},\psi(u,\dot{u})]$ on $TM$ tangent to lifts $U(t)$ analogously to~\cref{thm:invariants-node}. From our construction, and~\cref{thm:main-equiv-NODE}, it is clear that the solution $u(t)$ is $G$-equivariant, in the sense that $L_g u(t)$ is also a solution to $\ddot{u} = \psi(u,\dot{u})$, if and only if the flow $\Phi$ is $G$-equivariant. Consequently, the most general form of an augmented equivariant NODE can be deduced from~\cref{thm:general-sym-ODE}.

\begin{theorem}\label{thm:invariants-aug-node}
Given a complete set of second-order differential invariants $I_1,\ldots,I_{\mu_2}$ of $G$, the most general augmented $G$-equivariant NODE is (locally) of the form $H(I_1,\ldots,I_{\mu_2})=0$ where $H:\mathbb{R}^{\mu_2} \to \mathbb{R}^n$ is an arbitrary function.
\end{theorem}

In this way, the second-order differential invariants parameterise equivariant vector fields restricted to lifts, and equivalently the space of augmented equivariant manifold NODEs.

\begin{remark}
The geometric objects appearing in the equivariant augmentation can, as before, be understood in terms of jet bundles. Enforcing equivariance of $\phi$ corresponds to prolonging the action of $G$ to sections of the second jet bundle $J^{(2)}$, which are equivalent to lifts of $\Gamma_p$ to the tangent bundle of $TM$.
\end{remark}

To illustrate the above construction of equivariant augmentation and its universality properties, we return to the setting of~\cref{example:rot_eq_NODE} and consider the problem of learning a topologically non-trivial equivariant diffeomorphism.
\begin{example}\label{example:augmentation}
Let $M$ and $G$ be as in~\cref{example:rot_eq_NODE} and consider the augmented NODE $\dot{U}(t) = \phi_{U(t)}$, where $U(t)=\left[x(t),y(t),\dot{x}(t),\dot{y}(t)\right]$. To parameterise the space of such models equivariant under $G$ we consider second-order differential invariants of $G$. These are obtained from the second prolongation of the vector field $X$, which is given by
\begin{equation*}
    X^{(2)}=-y\partial_x+x\partial_y-\dot{y}\partial_{\dot{x}}+\dot{x}\partial_{\dot{y}}
    -\ddot{y}\partial_{\ddot{x}}+\ddot{x}\partial_{\ddot{y}}.
\end{equation*}
By solving the infinitesimal condition $X^{(2)}(I) = 0$ we obtain the first-order invariants $I_1,\ldots,I_4$ in~\cref{example:rot_eq_NODE} and the additional second-order invariants\footnote{Explicit expressions for $I_5$ and $I_6$ in terms of the Cartesian coordinates $[x,y]$ are suppressed for the sake of brevity.}
\begin{equation*}
    I_5=r^3\ddot{r} \,, \quad I_6=r^4\ddot{\theta}.
\end{equation*}

By~\cref{thm:invariants-aug-node}, any second-order $G$-equivariant ODE on $M$ can be written as $H(t,r,r\dot{r},r^2\dot{\theta},\ddot{r},\ddot{\theta})=0\in\mathbb{R}^2$ for some arbitrary function $H$. The equivalent expression in terms of $\ddot{u} = \psi(\dot{u},u)$ is given by 
\begin{equation*}
\begin{cases}
    \ddot{x}=A(t,r,r\dot{r},r^2\dot{\theta})\dot{x}-B(t,r,r\dot{r},r^2\dot{\theta})\dot{y}\\
    \ddot{y}=B(t,r,r\dot{r},r^2\dot{\theta})\dot{x}+A(t,r,r\dot{r},r^2\dot{\theta})\dot{y}
\end{cases},
\end{equation*}
where $A$ and $B$ are two arbitrary functions. Note that $\ddot{u}(t) = [\ddot{x}(t),\ddot{y}(t)]$ depends non-trivially on both $u$ and $\dot{u}$ through $A$ and $B$. In terms of~\eqref{eq:augmented-NODE}, this implies that the most general vector field $\phi$ in~\eqref{eq:augmented-NODE} restricted to lifts in $TM$ is
\begin{equation}\label{eq:aug_example}
\begin{aligned}
    \phi_{U(t)}=& \Bigl[ \dot{x}(t),\dot{y}(t),A(t,r,r\dot{r},r^2\dot{\theta})\dot{x}(t)-B(t,r,r\dot{r},r^2\dot{\theta})\dot{y}(t),\\
&B(t,r,r\dot{r},r^2\dot{\theta})\dot{x}(t)+A(t,r,r\dot{r},r^2\dot{\theta})\dot{y}(t)\Bigr].
\end{aligned}
\end{equation}

This is our desired parameterisation of equivariant vector fields in the augmented NODE, where $A$ and $B$ can again be approximated using feed-forward neural networks and the output is obtained by the projection $\pi$ onto the first two components of $U(t)$.

An example of rotationally equivariant diffeomorphism that cannot be approximated by a non-augmented NODE is the mapping $h:(r,\theta)\mapsto(1/r,\theta)$. We let the data set consist of 100 equally spaced points on $[-2,2]\times[-2,2]$ and use NODE models of the forms~\eqref{eq:rot_eq} and~\eqref{eq:aug_example} and train them to map the input points to their respective images under $h$. The performance of the non-augmented NODE and augmented NODE are compared in~\cref{fig:example_augmentation}. In the non-augmented case, the best performance is obtained when the points are transformed to a circle minimising the total distance to the target points. Crossing this circle is not possible due to the limitation of flow lines not being able to intersect. The augmented case, however, performs much better, with the output points almost overlapping the target points after $1000$ epochs of training.

\begin{figure}[htp]
\begin{subfigure}{\textwidth}
\centering
\includegraphics[height=4.1cm]{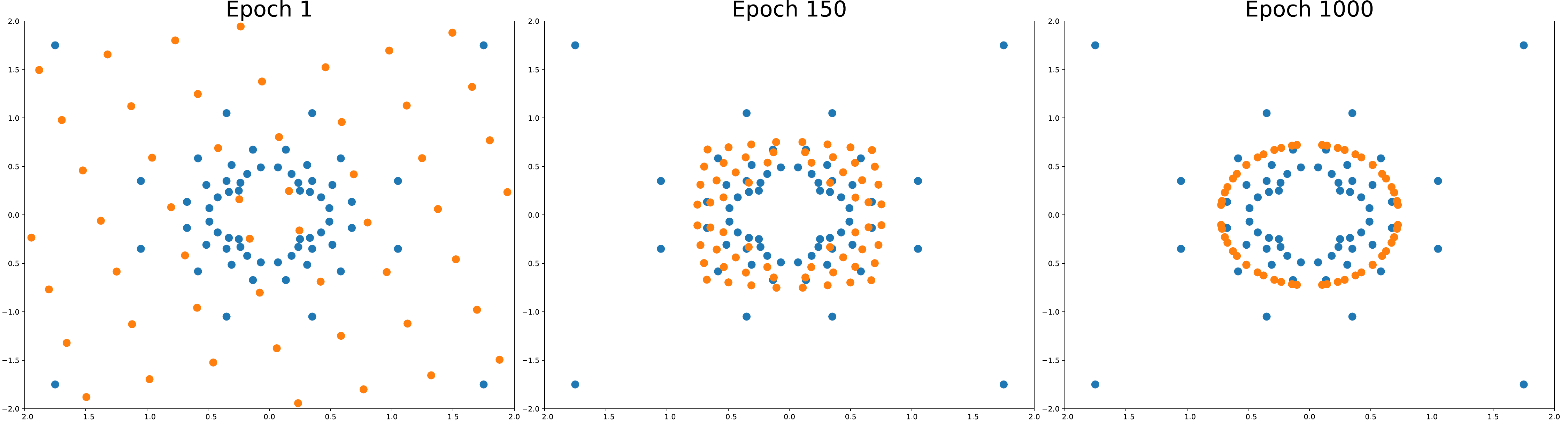}
\caption{Non-augmented case.}
\end{subfigure}

\bigskip

\begin{subfigure}{\textwidth}
\centering
\includegraphics[height=4.1cm]{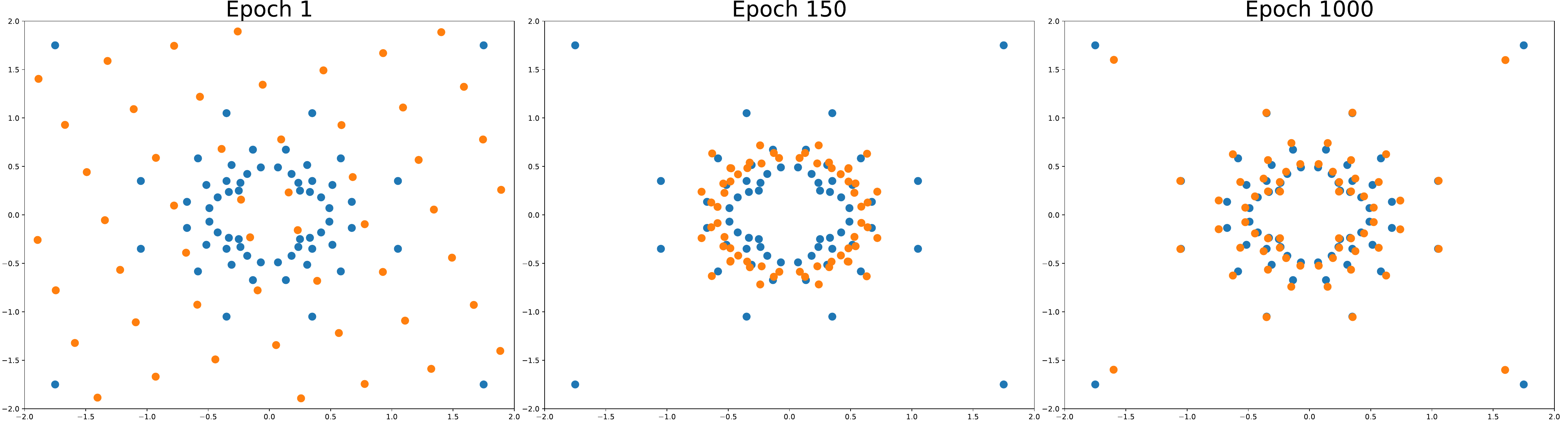}
\caption{Augmented case.}
\end{subfigure}
\caption{In \cref{example:augmentation}, we train a non-augmented NODE model and an augmented NODE model to approximate the equivariant diffeomorphism $h(r,\theta) = (1/r,\theta)$. The orange dots represent the resulting outputs by the models, while the blue dots represent the target values. In (a), the neural ODE has not been augmented. We see that, as a result of the solution curves not being able to cross, the outputs freeze on a circle minimising the distance to the targets after around 1000 epochs. In (b), the NODE model has been augmented to the tangent bundle. As a result, the outputs are able to better approximate the targets.
}
\label{fig:example_augmentation}
\end{figure}

To further illustrate the differences between the non-augmented and augmented model, we plot the flow lines from $t=0$ to $t=1$ for sample points with the same angular coordinate $\theta$.~\Cref{fig:flow-lines-non-aug} shows the flow lines $r(t)$ from $t=0$ to $t=1$ in the non-augmented case after 1000 epochs of training. We see that the curves appear to converge to a specific value of $r$ (around $r=0.8$), not intersecting. In~\cref{fig:flow-lines-aug}, we show the flow lines $r(t)$ and $\dot{r}(t)$ in the augmented case after 1000 epochs of training. We see that an intersection occurs in the left plot in~\cref{fig:flow-lines-aug}, while there are no intersections in the plot on the right-hand side.

\begin{figure}[!ht]
\centering
\includegraphics[scale=0.6]{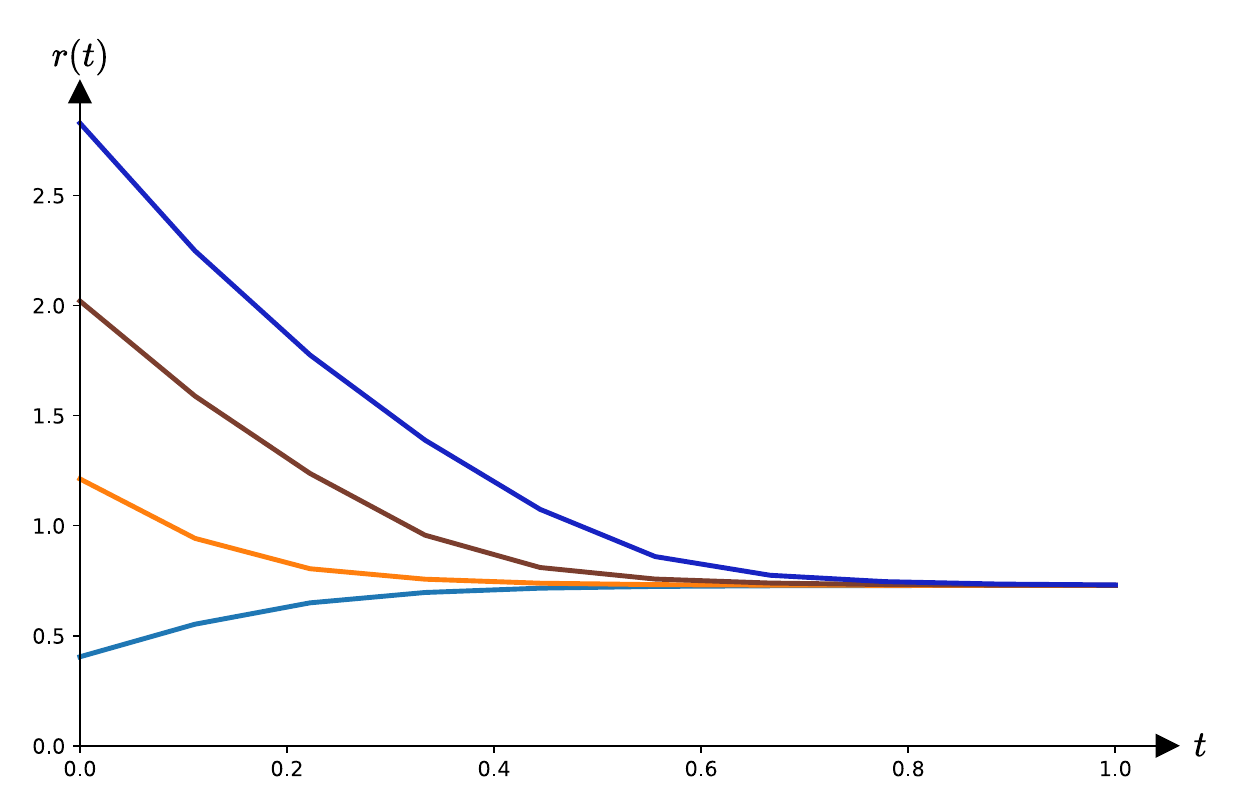}
\caption{Flow lines of the non-augmented NODE model in~\cref{example:augmentation} after 1000 epochs at $\theta\approx0.25\pi$. The curves converge to a specific value of $r$, but do not cross.}
\label{fig:flow-lines-non-aug}
\end{figure}

\begin{figure}[!ht]
\centering
\includegraphics[scale=0.45]{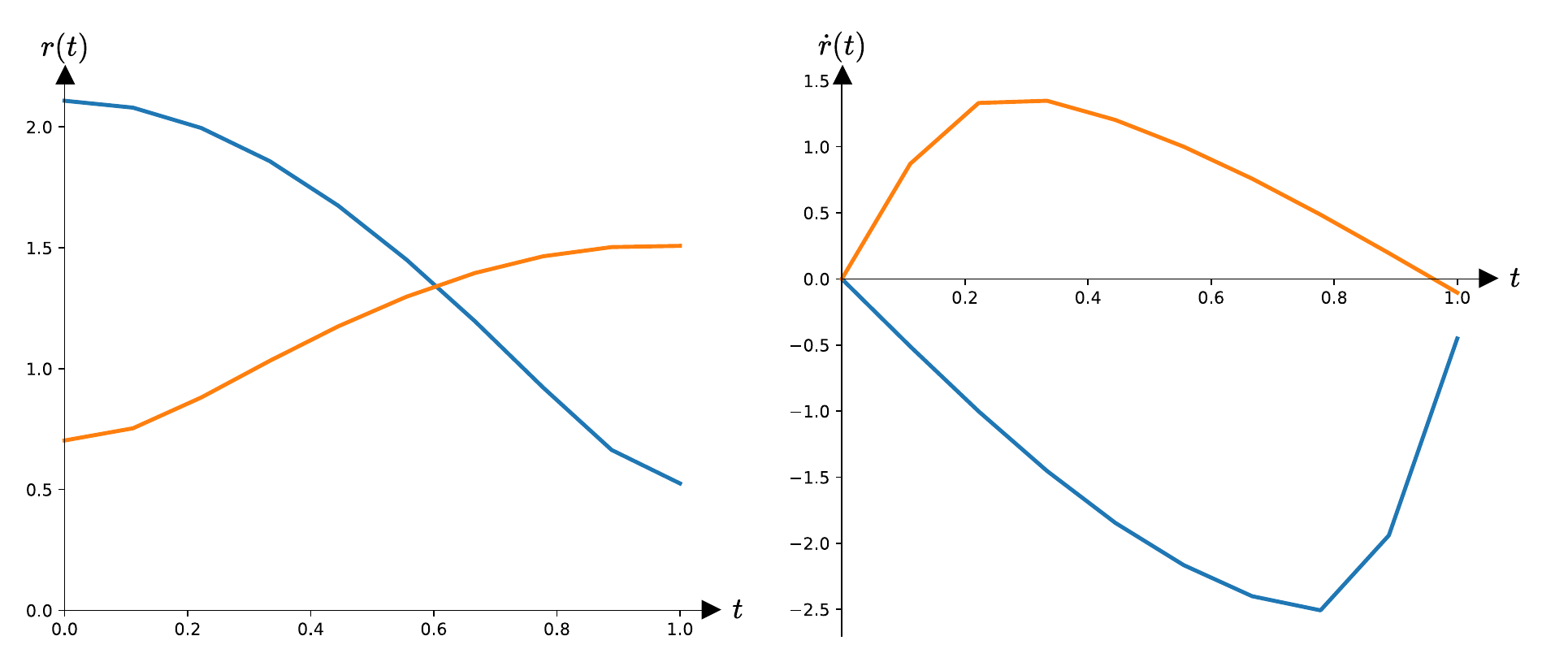}
\caption{Flow lines $r$ and $\dot{r}$ of the augmented NODE model in~\cref{example:augmentation} at $\theta\approx0.6\pi$. The model has been trained for 1000 epochs. Intersection in the left figure is allowed since the paths do not simultaneously cross in the figure to the right, which means that no intersections occur on the tangent bundle $TM$.
}
\label{fig:flow-lines-aug}
\end{figure}

This example illustrates the increased performance obtained using an augmented neural ODE to approximate topologically non-trivial diffeomorphisms.~\Cref{fig:flow-lines-aug} shows that lifting the solution curves to the tangent bundle $TM$ allows for intersections on the base space $M$, providing a more accurate solution.

\end{example}

\section{Equivariant fields and densities}\label{sec:eq-fields-and-den}
Having constructed the geometric framework for equivariant manifold NODEs and their augmentations, we now discuss how they are used to construct models acting on different objects on $M$, or feature maps in machine learning parlance. More specifically, this amounts to describing the induced action of the diffeomorphism $h : M \to M$ on such objects, and their properties under the action of $G$. We consider functions, densities and vector fields explicitly, but the principle can be applied to any kind of feature map.

\subsection{Equivariant functions and densities}

For a function $f:M\to\mathbb{R}$, or scalar field, the map induced by the diffeomorphism $h : M \to M$ is simply $f_h=f \circ h^{-1}$. Symmetry of a scalar function $f_h$ amounts to invariance, meaning $f(L_g p)=f(p)$. That equivariance of $h$ and $f$ implies equivariance of $f_h$ can be shown in one line as
\begin{equation*}
    f_h(L_g p)=f(h^{-1}(L_g p))=f(L_g h^{-1}(p))= f_h(p).
\end{equation*}

Next, we consider a (probability) density $\rho$ on $M$, in which case the induced action of $h$ is obtained as
\begin{equation}\label{eq:induced_prob}
    \rho_{h}(p)=\rho\left(h^{-1}(p)\right)\lvert\det{J_{h^{-1}}(p)}\rvert\,,
\end{equation}
where the Jacobian determinant accounts for the induced action on the volume element on $M$.

In~\cite{Kohler2020,Katsman2021} the authors show that, given a density $\rho$, invariant under the isometry group $G$, and a diffeomorphism $h:M \to M$ equivariant under $G$ (or some subgroup $H<G$), the transformed density $\rho_{h}$ is also $G$-invariant (or $H$-invariant). This makes it possible to construct continuous normalizing flows which preserve the symmetries of the densities.

The following theorem generalises the results on invariant densities in~\cite{Kohler2020,Katsman2021} to the case of arbitrary groups acting on $M$, which means that the Jacobian determinant corresponding to the change in the volume element is not necessarily unity. Equivariance under $G$ in the context of densities amounts to preservation of the probability associated with the infinitesimal volume element. Consequently, equivariance of $\rho$ is given by the following definition.

\begin{definition}
    The density $\rho$ on $M$ is equivariant if 
\begin{equation*}
    \rho(p)=\rho(L_g^{-1}p)\,\lvert{\det{J_{L_g^{-1}}(p)}}\rvert
\end{equation*}
for every $g \in G$ and $p \in M$. 
\end{definition}
\begin{theorem}\label{thm:induced-action-density}
Let $\rho$ be a $G$-equivariant density on $M$. If $h : M \to M$ is a $G$-equivariant diffeomorphism, then the induced density $\rho_h$ is $G$-equivariant.

\end{theorem}
\begin{proof}
Equivariance of $h$ means $L_g\circ h=h\circ L_g$ and it is straightforward to show that $h^{-1}$ is also $G$-equivariant. From this, the chain rule and the equivariance of $\rho$, it follows that
    \begin{equation}\label{eq:density-pf-1}
    \rho(h^{-1}(p))
    =
    \rho(h^{-1}L_g^{-1}(p))
    \,\lvert{\det{J_{L_g^{-1}}(h^{-1}(p))}}\rvert.
    \end{equation}
By a rearrangement of~\eqref{eq:induced_prob},
\begin{equation}\label{eq:density-pf-2}
        \rho(h^{-1}L_g^{-1}(p))
    =
    \rho_h(L_g^{-1}p)\frac{1}{\lvert{\det{J_{h^{-1}}(L_g^{-1}p)}}\rvert}.
\end{equation}
Finally, we use~\eqref{eq:induced_prob}, combine~\eqref{eq:density-pf-1} and~\eqref{eq:density-pf-2} and use the chain rule to obtain
\begin{align*}
    \rho_h(p)=
    \rho_h(L_g^{-1}p)\frac{\lvert{\det{J_{L_g^{-1}\circ h^{-1}}(p)}}\rvert\,\lvert{\det{J_{L_g^{-1}}(p)}}\rvert}{\lvert{\det{J_{h^{-1}\circ L_g^{-1}}(p)}}\rvert}=
    \rho_h(L_g^{-1}p)\,\lvert{\det{J_{L_g^{-1}}(p)}}\rvert.
\end{align*}
Thus, $\rho_h(p)=\rho_h(L_g^{-1}p)\,\lvert{\det{J_{L_g^{-1}}(p)}}\rvert$ for every $g \in G$ and $p \in M$, and consequently $\rho_h$ is equivariant.
\end{proof}

\subsection{Equivariant vector fields}

Next, we consider a vector field $V:M\to TM$ and the induced action of the diffeomorphism $h:M\to M$ obtained as the push-forward of $V$ along $h$. At the point $p \in M$, the transformed vector field $V_h$ is given by
\begin{equation*}
    V_h(p)[f] = V(h^{-1}(p))[f \circ h],
\end{equation*}
where $f$ is an arbitrary differentiable function defined on a neighbourhood of $h^{-1}(p)$. Note that for clarity we use $V(p)$ to denote the evaluation of $V$ at $p$ in this section. Equivariance of a vector field $V$ amounts to left invariance under $G$ in the usual sense, i.e., $V(L_gp)=(L_g)_*V(p)$ for every $g \in G$ and $p \in M$, see~\cref{dfn:vector_field_equivariance}. Similarly to the objects previously considered, the vector field $V_h$ inherits the equivariance of $V$ and $h$.

\begin{theorem}\label{thm:induced-action-vector}
Let $V$ be a $G$-equivariant vector field on $M$, and $h :M\to M$ be $G$-equivariant diffeomorphism. Then $V_h$ is $G$-equivariant.
\end{theorem}
\begin{proof}
We want to show that $V_h(L_gp)=(L_g)_*(V_h(p))$ for all $g \in G$ and $p\in M$. Let $f$ be an arbitrary real function defined on a neighbourhood of $h^{-1}(p)$. It follows by the definition of the push-forward and the equivariance of $h^{-1}$ that
\begin{align*}
    V_h(L_gp)[f] &= V(L_g(h^{-1}(p)))[f\circ h].
\end{align*}
By the equivariance of $V$ and the definition of the push-forward, this in turn equals $(L_g)_*V_h(p)[f]$. Since $f$ is arbitrary, $V_\Phi(L_gp)=(L_g)_*V_\Phi(p)$ for all $p\in M$ and $g \in G$, which completes the proof.
\end{proof}

To illustrate the principle of the induced action and the way equivariance of $h : M \to M$ preserves the symmetry of the object being transformed, we visualise the different quantities appearing in~\cref{thm:induced-action-vector} in~\cref{fig:induced-action-vector}.

\begin{figure}[!ht]
\centering
\includegraphics[scale=0.7]{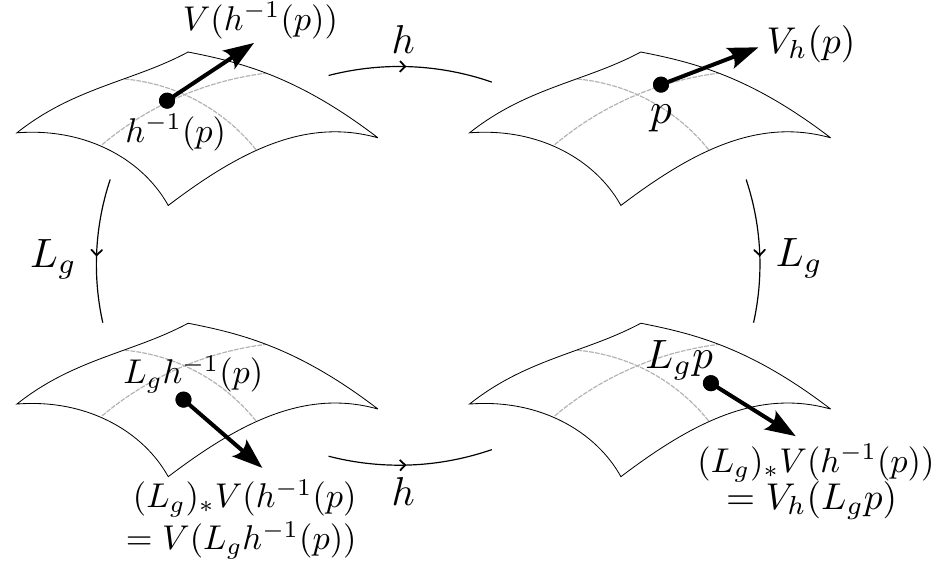}
\caption{Visualisation of the constructions in~\cref{thm:induced-action-vector}.}
\label{fig:induced-action-vector}
\end{figure}

The results in this section allow for the construction of $G$-equivariant models based on NODEs for the different kinds of geometric objects on $M$ considered. The approach uses the diffeomorphism $h: M \to M$ to induce a transformation, which defines the action of the model and is possible to generalise to other types of differential geometric objects. In particular, the induced action on vectors generalises to tensor fields on $M$.

\section{Conclusion}
In this paper, we develop a novel geometric framework for equivariant manifold neural ODEs based on the classical theory of symmetries of differential equations. In particular, for any smooth, connected manifold $M$ and connected Lie symmetry group $G$ acting semi-regularly on $M$ we establish the equivalence of the different notions of equivariance related to the Cauchy problem and show how the space of equivariant NODEs can be efficiently parameterised using the differential invariants of $G$. This generalises the idea in~\cite{Katsman2021} to use gradients of invariant functions by exploiting the fundamental role of differential invariants in the theory of symmetries of differential equations.

Subsequently, we show how the manifold $M$, and the neural ODEs, can be augmented to obtain models which are universal approximators of diffeomorphisms $h: M \to M$. Our construction extends previous work in the Euclidean case~\cite{Zhang2020}, to the manifold setting and provides the appropriate geometric framework to establish universality of augmented manifold NODE models meaning that they are capable of learning any  diffeomorphism $h:M \to M$. Furthermore, we show that the augmentation is equivariant with respect to a non-trivial symmetry group $G$ acting on the manifold and that universality persists, meaning that augmented equivariant manifold NODEs are universal approximators of equivariant diffeomorphisms $h:M \to M$. In the equivariant case, we also show how the augmentation fits in our geometric framework by parameterising the space of augmented equivariant manifold NODEs using second-order differential invariants.

Finally, following previous works on manifold NODEs, we discuss how different kinds of feature maps can be modelled using the induced action of $h$ on different types of fields on $M$. As a prominent example of learning fields beyond densities, we consider explicitly vector fields found in many physical applications. Our framework could, for example, be used for learning vector fields on the sphere, which has applications, e.g., in omnidirectional vision~\cite{Coors2018}, weather forecasting and cosmological applications~\cite{Perraudin2020}.

In our formulation, all calculations involving the group $G$ are concerned with determining the differential invariants $I_1,\ldots,I_{\mu}$ which is analytically tractable due to the fact that $X^{(k)}$ is a linear operator. There are several computer algebra software packages dedicated to symmetries of differential equations which can be used to perform these calculations. The invariants are computed prior to implementation and training and require no discretisation of the group $G$. Compared to group equivariant CNNs, in our framework integration over $G$ or $G/H$ in the convolutional layers is replaced by computing derivatives on $J^{(1)}$ (see Prop. 2.53 in \cite{Olver1993}).

It is interesting to note that the notion of symmetry transformations of differential equations typically only requires a local action of $G$, as compared to the global action on homogeneous spaces $M=G/H$ emphasised in \cite{Kondor2018}. In particular, in contrast to the group equivariant networks, there is no integration over the group $G$. In fact, the only computations using the group $G$ required during training and inference are straightforward function evaluations of the invariants. Since the invariants only depend on the manifold and symmetry group, they only need to be computed once for each combination of $M$ and $G$ and can then be applied in NODE models of any dataset with the corresponding differentiable structure and symmetry group.

The structure of the invariants $I_1,\ldots,I_{\mu}$ can be used to impose further constraints on the model by excluding a subset of the $I_k$ in the construction of the ODE $\dot{u}(t) = \phi(I_1,\ldots,I_{\mu})$. This corresponds to enlarging the symmetry group and the constraints are 'physically' meaningful since they are invariant under the action of $G$. A very interesting direction for future work would be to investigate this approach systematically and more generally explore the connections between invariants in our geometric framework to conservation laws in the physical sciences~\cite{Greydanus2019,Lutter2019,Toth2020,Cranmer2020}.

To understand the merits and limitations compared to other network architectures, the performance that can be obtained by equivariant manifold NODE models -- both in our framework and those of others -- in realistic applications requires significant further investigation. In this work, we focus on the mathematical foundations of equivariant manifold NODEs and their augmentation, but there is a substantial body of experimental results in the literature (see~\cref{sec:related-work}) showing that NODE models are practically feasible and viable as approximators. In future work, we plan to continue the exploration of the modelling capabilities of NODEs for different types of applications and feature maps (or fields) on topologically non-trivial manifolds $M$ with symmetries.

\section*{Acknowledgements}
We thank Jan Gerken for helpful discussions. The work of E.A., D.P., and F.O. was supported by the Wallenberg AI, Autonomous Systems and Software Program (WASP) funded by the Knut and Alice Wallenberg Foundation.

\bibliographystyle{unsrtnat}
\bibliography{cites}

\begin{thebibliography}{49}
\providecommand{\natexlab}[1]{#1}
\providecommand{\url}[1]{\texttt{#1}}
\expandafter\ifx\csname urlstyle\endcsname\relax
  \providecommand{\doi}[1]{doi: #1}\else
  \providecommand{\doi}{doi: \begingroup \urlstyle{rm}\Url}\fi

\bibitem[{Bronstein} et~al.(2017){Bronstein}, {Bruna}, {LeCun}, {Szlam}, and {Vandergheynst}]{Bronstein2017}
Michael~M. {Bronstein}, Joan {Bruna}, Yann {LeCun}, Arthur {Szlam}, and Pierre {Vandergheynst}.
\newblock Geometric deep learning: {G}oing beyond {E}uclidean data.
\newblock \emph{IEEE Signal Processing Magazine}, 34\penalty0 (4):\penalty0 18--42, 2017.

\bibitem[Bronstein et~al.(2021)Bronstein, Bruna, Cohen, and Velickovic]{Bronstein2021}
Michael~M. Bronstein, Joan Bruna, Taco Cohen, and Petar Velickovic.
\newblock Geometric deep learning: Grids, groups, graphs, geodesics, and gauges.
\newblock \emph{Arxiv e-prints arXiv:2104.13478.}, 2021.

\bibitem[Gerken et~al.(2023)Gerken, Aronsson, Carlsson, Linander, Ohlsson, Petersson, and Persson]{Gerken2020}
Jan~E. Gerken, Jimmy Aronsson, Oscar Carlsson, Hampus Linander, Fredrik Ohlsson, Christoffer Petersson, and Daniel Persson.
\newblock Geometric deep learning and equivariant neural networks.
\newblock \emph{Artificial Intelligence Review}, 56:\penalty0 14605–14662, 2023.

\bibitem[Chen et~al.(2018)Chen, Rubanova, Bettencourt, and Duvenaud]{Chen2018}
Ricky T.~Q. Chen, Yulia Rubanova, Jesse Bettencourt, and David~K. Duvenaud.
\newblock Neural ordinary differential equations.
\newblock \emph{Advances in Neural Information Processing Systems}, 31, 2018.

\bibitem[He et~al.(2016)He, Zhang, Ren, and Sun]{He2016}
Kaiming He, Xiangyu Zhang, Shaoqing Ren, and Jian Sun.
\newblock Deep residual learning for image recognition.
\newblock In \emph{Proceedings of the 2016 IEEE Conference on Computer Vision and Pattern Recognition (CVPR)}, pages 770--778, 2016.

\bibitem[E(2017)]{E2017}
Weinan E.
\newblock A proposal on machine learning via dynamical systems.
\newblock \emph{Commun. Math. Stat.}, 5:\penalty0 1--11, 2017.

\bibitem[Haber and Ruthotto(2018)]{Haber2018}
Eldad Haber and Lars Ruthotto.
\newblock Stable architectures for deep neural networks.
\newblock \emph{Inverse Problems}, 34:\penalty0 014004, 2018.

\bibitem[Ruthotto and Haber(2020)]{Ruthotto2020}
Lars Ruthotto and Eldad Haber.
\newblock Deep neural networks motivated by partial differential equations.
\newblock \emph{J. Math. Imaging Vis}, 62:\penalty0 352–364, 2020.

\bibitem[Stapor et~al.(2018)Stapor, Fröhlich, and Hasenauer]{stapor2018}
Paul Stapor, Fabian Fröhlich, and Jan Hasenauer.
\newblock Optimization and profile calculation of {ODE} models using second order adjoint sensitivity analysis.
\newblock \emph{Bioinformatics}, 34:\penalty0 151, 2018.

\bibitem[Rezende and Mohamed(2015)]{Rezende2015}
Danilo~Jimenez Rezende and Shakir Mohamed.
\newblock Variational inference with normalizing flows.
\newblock In \emph{Proceedings of the 32nd International Conference on Machine Learning}, volume~37, pages 1530--1538. PMLR, 2015.

\bibitem[Falorsi and Forré(2020)]{Falorsi2020}
Luca Falorsi and Patrick Forré.
\newblock Neural ordinary differential equations on manifolds.
\newblock In \emph{Proceedings of the INNF+ Workshop of the International Conference on Machine Learning (ICML)}, 2020.

\bibitem[Lou et~al.(2020)Lou, Lim, Katsman, Huang, Jiang, Lim, and Sa]{Lou2020}
Aaron Lou, Derek Lim, Isay Katsman, Leo Huang, Qingxuan Jiang, Ser~Nam Lim, and Christopher M.~De Sa.
\newblock Neural manifold ordinary differential equations.
\newblock \emph{Advances in Neural Information Processing Systems}, 33, 2020.

\bibitem[Mathieu and Nickel(2020)]{Mathieu2020}
Emile Mathieu and Maximilian Nickel.
\newblock Riemannian continuous normalizing flows.
\newblock \emph{Advances in Neural Information Processing Systems}, 33, 2020.

\bibitem[Köhler et~al.(2020)Köhler, Klein, and Noé]{Kohler2020}
Jonas Köhler, Leon Klein, and Frank Noé.
\newblock Equivariant flows: {E}xact likelihood generative learning for symmetric densities.
\newblock In \emph{Proceedings of the 37th International Conference on Machine Learning}, pages 5361--5370. PMLR, 2020.

\bibitem[Katsman et~al.(2021)Katsman, Lou, Lim, Jiang, Lim, and Sa]{Katsman2021}
Isay Katsman, Aaron Lou, Derek Lim, Qingxuan Jiang, Ser-Nam Lim, and Christopher~De Sa.
\newblock Equivariant manifold flows.
\newblock In \emph{Advances in Neural Information Processing Systems}, volume~34, 2021.

\bibitem[Chen et~al.(2019)Chen, Behrmann, Duvenaud, and Jacobsen]{Chen2019}
Ricky T.~Q. Chen, Jens Behrmann, David Duvenaud, and Jörn-Henrik Jacobsen.
\newblock Residual flows for invertible generative modeling.
\newblock \emph{Advances in Neural Information Processing Systems}, 32, 2019.

\bibitem[Grathwohl et~al.(2019)Grathwohl, Chen, Bettencourt, Sutskever, and Duvenaud]{Grathwohl2019}
Will Grathwohl, Ricky T.~Q. Chen, Jesse Bettencourt, Ilya Sutskever, and David Duvenaud.
\newblock {FFJORD}: Free-form continuous dynamics for scalable reversible generative models.
\newblock In \emph{Proceedings of the 7th International Conference on Learning Representations (ICLR 2019)}, 2019.

\bibitem[Onken et~al.(2021)Onken, Fung, Li, and Ruthotto]{Onken2021}
Derek Onken, Samy~Wu Fung, Xingjian Li, and Lars Ruthotto.
\newblock {OT-Flow}: {F}ast and accurate continuous normalizing flows via optimal transport.
\newblock In \emph{Proceedings of the 35th AAAI Conference on Artificial Intelligence (AAAI-21)}, 2021.

\bibitem[Satorras et~al.(2021)Satorras, Hoogeboom, Fuchs, Posner, and Welling]{GarciaSatorras2021}
Victor~Garcia Satorras, Emiel Hoogeboom, Fabian~B. Fuchs, Ingmar Posner, and Max Welling.
\newblock {E}(n) equivariant normalizing flows.
\newblock \emph{Advances in Neural Information Processing Systems}, 34, 2021.

\bibitem[Köhler et~al.(2019)Köhler, Klein, and Noé]{Kohler2019}
Jonas Köhler, Leon Klein, and Frank Noé.
\newblock Equivariant flows: sampling configurations for multi-body systems with symmetric energies.
\newblock In \emph{Proceedings of the Second Workshop on Machine Learning and the Physical Sciences (NeurIPS 2019)}, 2019.

\bibitem[Rezende et~al.(2019)Rezende, Racanière, Higgins, and Toth]{Rezende2019}
Danilo~Jimenez Rezende, Sébastien Racanière, Irina Higgins, and Peter Toth.
\newblock Equivariant hamiltonian flows.
\newblock In \emph{Proceedings of the Second Workshop on Machine Learning and the Physical Sciences (NeurIPS 2019)}, 2019.

\bibitem[Rezende et~al.(2020)Rezende, Papamakarios, Racanière, Albergo, Kanwar, Shanahan, and Cranmer]{Rezende2020}
Danilo~Jimenez Rezende, George Papamakarios, Sébastien Racanière, Michael~S. Albergo, Gurtej Kanwar, Phiala~E. Shanahan, and Kyle Cranmer.
\newblock Normalizing flows on tori and spheres.
\newblock \emph{Arxiv e-print arXiv:2002.02428}, 2020.

\bibitem[Katsman et~al.(2023)Katsman, Chen, Holalkere, Asch, Lou, Lim, and Sa]{Katsman2023}
Isay Katsman, Eric Chen, Sidhanth Holalkere, Anna Asch, Aaron Lou, Ser-Nam Lim, and Christopher~De Sa.
\newblock Riemannian residual neural networks.
\newblock In \emph{Advances in Neural Information Processing Systems}, 2023.

\bibitem[Utz(1981)]{Utz1981}
W.~Roy Utz.
\newblock The embedding of homeomorphisms in continuous flows.
\newblock \emph{Topology Proceedings}, 6:\penalty0 159--177, 1981.

\bibitem[Dupont et~al.(2019)Dupont, Doucet, and Teh]{Dupont2019}
Emilien Dupont, Arnaud Doucet, and Yee~Whye Teh.
\newblock Augmented neural {ODE}s.
\newblock In \emph{Advances in Neural Information Processing Systems}, volume~32, 2019.

\bibitem[Zhang et~al.(2020)Zhang, Gao, Unterman, and Arodz]{Zhang2020}
Han Zhang, Xi~Gao, Jacob Unterman, and Tom Arodz.
\newblock Approximation capabilities of neural {ODE}s and invertible residual networks.
\newblock In \emph{Proceedings of the 37th International Conference on Machine Learning}, volume 119, pages 11086--11095. PMLR, 2020.

\bibitem[Bose et~al.(2021)Bose, Brubaker, and Kobyzev]{Bose2021}
Avishek~Joey Bose, Mila~Marcus Brubaker, and Ivan Kobyzev.
\newblock Equivariant finite normalizing flows.
\newblock \emph{Arxiv e-prints arXiv:2110.08649}, 2021.

\bibitem[Olver(1993)]{Olver1993}
Peter~J. Olver.
\newblock \emph{Applications of {Lie} Groups to Differential Equations}.
\newblock Springer-Verlag, 1993.

\bibitem[Knibbeler(2024)]{Knibbeler2024}
Vincent Knibbeler.
\newblock Computing equivariant matrices on homogeneous spaces for geometric deep learning and automorphic {Lie} algebras.
\newblock \emph{Advances in Computaitonal Mathematics}, 50, 2024.

\bibitem[Kondor and Trivedi(2018)]{Kondor2018}
Risi Kondor and Shubhendu Trivedi.
\newblock On the generalization of equivariance and convolution in neural networks to the action of compact groups.
\newblock In \emph{Proceedings of the 35th International Conference on Machine Learning}, pages 2747--2755. PMLR, 2018.

\bibitem[{Cohen} et~al.(2019){Cohen}, {Weiler}, {Kicanaoglu}, and {Welling}]{Cohen2019}
Taco~S. {Cohen}, Maurice {Weiler}, Berkay {Kicanaoglu}, and Max {Welling}.
\newblock Gauge equivariant convolutional networks and the icosahedral {CNN}.
\newblock In \emph{Proceedings of the 36th International Conference on Machine Learning}, volume~97, pages 1321--1330. PMLR, 2019.

\bibitem[Aronsson(2022)]{Aronsson2022}
Jimmy Aronsson.
\newblock Homogeneous vector bundles and {G}-equivariant convolutional neural networks.
\newblock \emph{Sampling Theory, Signal Processing, and Data Analysis}, 20, 2022.

\bibitem[Akhound-Sadegh et~al.(2023)Akhound-Sadegh, Perreault-Levasseur, Brandstetter, Welling, and Ravanbakhsh]{AkhoundSadegh2023}
Tara Akhound-Sadegh, Laurence Perreault-Levasseur, Johannes Brandstetter, Max Welling, and Siamak Ravanbakhsh.
\newblock Lie point symmetries and physics informed networks.
\newblock \emph{Advances in Neural Information Processing Systems}, 36:\penalty0 42468--42481, 2023.

\bibitem[Arora et~al.(2024)Arora, Bihlo, and Valiquette]{Arora2024}
Shivam Arora, Alex Bihlo, and Francis Valiquette.
\newblock Invariant physics-informed neural networks for ordinary differential equations.
\newblock \emph{Journal of Machine Learning Research}, 25:\penalty0 1--24, 2024.

\bibitem[Lagrave and Tron(2022)]{Lagrave2022}
Pierre-Yves Lagrave and Eliot Tron.
\newblock Equivariant neural networks and differential invariants theory for solving partial differential equations.
\newblock \emph{Physical Sciences Forum}, 5, 2022.

\bibitem[Lipman et~al.(2023)Lipman, Chen, Ben-Hamu, Nickel, and Le]{Lipman2023}
Yaron Lipman, Ricky T.~Q. Chen, Heli Ben-Hamu, Maximilian Nickel, and Matt Le.
\newblock Flow matching for generative modeling.
\newblock In \emph{Proceedings of the 11th International Conference on Learning Representations (ICLR)}, 2023.

\bibitem[Tong et~al.(2024)Tong, Fatras, Malkin, Huguet, Zhang, Rector-Brooks, Wolf, and Bengio]{Tong2024}
Alexander Tong, Kilian Fatras, Nikolay Malkin, Guillaume Huguet, Yanlei Zhang, Jarrid Rector-Brooks, Guy Wolf, and Yoshua Bengio.
\newblock Improving and generalizing flow-based generative models with minibatch optimal transport.
\newblock \emph{Transactions of Machine Learning Research}, 03, 2024.

\bibitem[Chen and Lipman(2024)]{Chen2024}
Ricky T.~Q. Chen and Yaron Lipman.
\newblock Flow matching on general geometries.
\newblock In \emph{Proceedings of the 12th International Conference on Learning Representations (ICLR)}, 2024.

\bibitem[Yim et~al.(2023)Yim, Campbell, Foong, Gastegger, Jiménez-Luna, Lewis, Satorras, Veeling, Barzilay, Jaakkola, and Noé]{Yim2023}
Jason Yim, Andrew Campbell, Andrew Y.~K. Foong, Michael Gastegger, José Jiménez-Luna, Sarah Lewis, Victor~Garcia Satorras, Bastiaan~S. Veeling, Regina Barzilay, Tommi Jaakkola, and Frank Noé.
\newblock Fast protein backbone generation with {SE(3)} flow matching.
\newblock In \emph{Proceedings of the Machine Learning in Structural Biology Workshop (NeurIPS 2023)}, 2023.

\bibitem[Bose et~al.(2024)Bose, Akhound-Sadegh, Huguet, Fatras, Rector-Brooks, Liu, Nica, Korablyov, Bronstein, and Tong]{Bose2024}
Avishek~(Joey) Bose, Tara Akhound-Sadegh, Guillaume Huguet, Kilian Fatras, Jarrid Rector-Brooks, Cheng-Hao Liu, Andrei~Cristian Nica, Maksym Korablyov, Michael Bronstein, and Alexander Tong.
\newblock {SE(3)} stochastic flow matching for protein backbone generation.
\newblock In \emph{Proceedings of the 12th International Conference on Learning Representations (ICLR)}, 2024.

\bibitem[Klein et~al.(2023)Klein, Krämer, and Noé]{Klein2023}
Leon Klein, Andreas Krämer, and Frank Noé.
\newblock Equivariant flow matching.
\newblock In \emph{Advances in Neural Information Processing Systems}, 2023.

\bibitem[Olver(1995)]{Olver1995}
Peter~J. Olver.
\newblock \emph{Equivalence, Invariants and Symmetry}.
\newblock Cambridge University Press, 1995.

\bibitem[Nakahara(2003)]{Nakahara2003}
Mikio Nakahara.
\newblock \emph{Geometry, Topology and Physics}.
\newblock CRC Press, 2003.

\bibitem[Coors et~al.(2018)Coors, Condurache, and Geiger]{Coors2018}
Benjamin Coors, Alexandru~Paul Condurache, and Andreas Geiger.
\newblock {SphereNet}: {L}earning spherical representations for detection and classification in omnidirectional images.
\newblock In \emph{Proceedings of the 15th European Conference on Computer Vision (ECCV)}, 2018.

\bibitem[Defferrard et~al.(2020)Defferrard, Milani, Gusset, and Perraudin]{Perraudin2020}
Michaël Defferrard, Martino Milani, Frédérick Gusset, and Nathanaël Perraudin.
\newblock {D}eep{S}here: {A} graph-based spherical {CNN}.
\newblock In \emph{Proceedings of the 8th International Conference on Learning Representations (ICLR)}, 2020.

\bibitem[Greydanus et~al.(2019)Greydanus, Dzamba, and Yosinski]{Greydanus2019}
Sam Greydanus, Misko Dzamba, and Jason Yosinski.
\newblock Hamiltonian neural networks.
\newblock In \emph{Advances in Neural Information Processing Systems}, volume~32, 2019.

\bibitem[Lutter et~al.(2019)Lutter, Ritter, and Peters]{Lutter2019}
Michael Lutter, Christian Ritter, and Jan Peters.
\newblock Deep lagrangian networks: {U}sing physics as model prior for deep learning.
\newblock In \emph{Proceedings of the 8th International Conference on Learning Representations (ICLR)}, 2019.

\bibitem[Toth et~al.(2020)Toth, Rezende, Jaegle, Racanière, Botev, and Higgins]{Toth2020}
Peter Toth, Danilo~Jimenez Rezende, Andrew Jaegle, Sébastien Racanière, Aleksandar Botev, and Irina Higgins.
\newblock Hamiltonian generative networks.
\newblock In \emph{Proceedings of the 8th International Conference on Learning Representations (ICLR)}, 2020.

\bibitem[Cranmer et~al.(2020)Cranmer, Greydanus, Hoyer, Battaglia, Spergel, and Ho]{Cranmer2020}
Miles Cranmer, Sam Greydanus, Stephan Hoyer, Peter Battaglia, David Spergel, and Shirley Ho.
\newblock Lagrangian neural networks.
\newblock In \emph{Proceedings of the Deep Differential Equations Workshop, ICLR 2020}, 2020.

\end{thebibliography}

\end{document}